\documentclass{article} 
\usepackage{iclr2027_conference,times}

\usepackage{amsmath,amsfonts,bm}

\def\eqref#1{equation~\ref{#1}}

\def\1{\bm{1}}

\DeclareMathAlphabet{\mathsfit}{\encodingdefault}{\sfdefault}{m}{sl}
\SetMathAlphabet{\mathsfit}{bold}{\encodingdefault}{\sfdefault}{bx}{n}

\def\gX{{\mathcal{X}}}

\DeclareMathOperator*{\argmin}{arg\,min}

\usepackage{hyperref}
\usepackage{url}
\usepackage{tikz}
\usepackage{amsmath, amssymb}
\usetikzlibrary{positioning, arrows.meta, shapes.geometric, calc, backgrounds, fit, scopes}

\usepackage{enumitem}
\usepackage{caption}
\usepackage{wasysym}
\usepackage{tabularx}
\usepackage{booktabs}
\usepackage{wrapfig}
\usepackage{titletoc}     

\usepackage{amsthm}

\theoremstyle{plain}
\newtheorem{theorem}{Theorem}[section]
\newtheorem{proposition}[theorem]{Proposition}

\newtheorem{corollary}[theorem]{Corollary}
\theoremstyle{definition}
\newtheorem{definition}[theorem]{Definition}

\newtheorem{remark}[theorem]{Remark}
\theoremstyle{remark}

\title{Discrete Action Matching: Learning Stochastic Dynamics from Samples via State Graphs}

\author{Mikhail Persiianov \\
Applied AI Institute\\
Moscow, Russia \\
\texttt{persiianov.mi@gmail.com} \\
\And
Alexander Korotin \\
Applied AI Institute \\
AXXX \\
Moscow, Russia \\
}

\iclrfinalcopy 
\begin{document}

\maketitle

\begin{abstract}
Learning population dynamics from unpaired temporal marginals is an ill-posed inverse problem that requires structural assumptions on the underlying dynamics. We introduce \emph{Discrete Action Matching} (DAM), a finite-state counterpart of Action Matching based on discrete Wasserstein geometry. For a prescribed marginal path and transport geometry, we derive an action-minimization objective for its canonical minimum-kinetic-energy current. Our key observation is that the density dependence of the discrete action reduces to neighboring density ratios. Along an empirical interpolation of the snapshots, DAM first estimates these ratios and then learns an action potential. The learned fields also define a graph-supported Markov sampler. Experiments on controlled synthetic dynamics and real mouse gastrulation data evaluate marginal reconstruction and interpolation. Additional experiments approximate numerical surface-transport paths from paired samples.
\end{abstract}

\noindent\begin{minipage}{\linewidth}
    \centering
    \includegraphics[width=\linewidth]{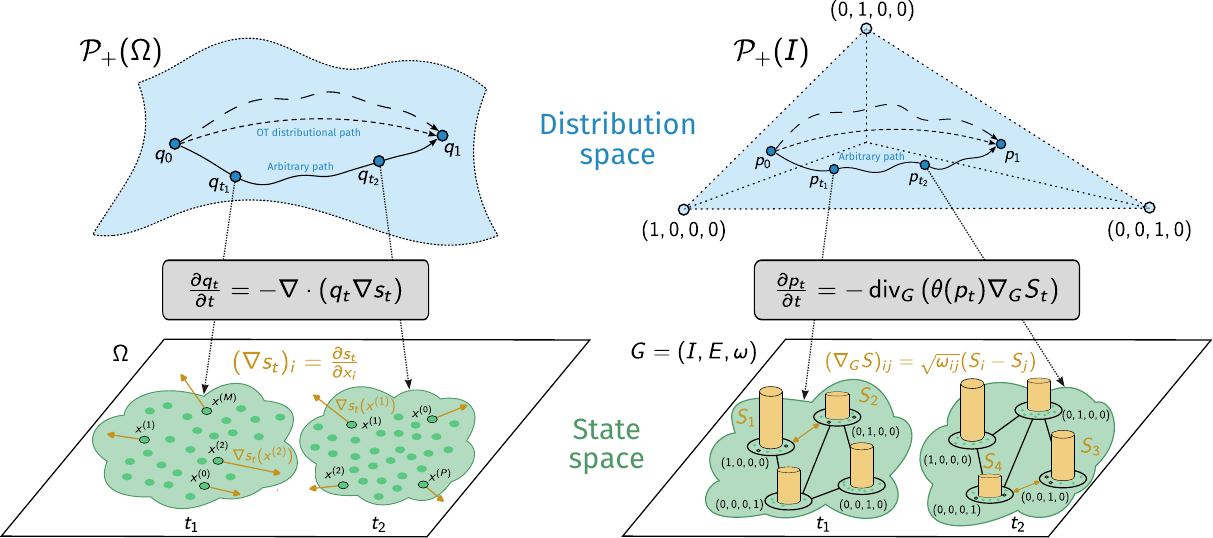}
    \captionsetup{font=small,hypcap=false}
    \captionof{figure}{
        \textbf{Action Matching \citep[left]{neklyudov2023action} and Discrete Action Matching (ours, right).}
        Continuity equations relate prescribed marginal paths (top) to potential fields and samples in the corresponding state spaces (bottom).
        Dashed curves illustrate alternative paths.
        Notation is introduced at the end of Section~\ref{sec:introduction}.
    }
    \label{fig:dam_pipeline}
\end{minipage}
\par\medskip

\section{Introduction}
\label{sec:introduction}

Learning population dynamics from temporal snapshots is a ubiquitous inverse problem arising when populations can be observed repeatedly but individual trajectories cannot be tracked \citep{tong2020trajectorynet, lavenant2021towards}. In temporal single-cell data, for example, successive measurements describe changing cell populations without revealing the transitions of individual cells \citep{schiebinger2019optimal}. Each snapshot provides samples from a population distribution, while the motion connecting these distributions remains unobserved. Learning dynamics from such data therefore requires a principle \citep{guan2025gradient} for selecting among compatible motions.

Optimal-transport geometry \citep{lott2008some} provides a natural selection principle: among all motions producing the same instantaneous change in a distribution, choose the one with minimum kinetic energy \citep{benamou2000computational,ambrosio2008gradient}. Action Matching (AM) \citep{neklyudov2023action} turns this principle into a learning objective for a prescribed marginal path on a continuous state space. It recovers the canonical velocity field from marginal samples without observing velocities or individual trajectories. Subsequent work extends this variational perspective to optimizing marginal paths \citep{neklyudov2024wasserstein} and learning population mechanics \citep{guan2026call}.

Many population models describe individuals as occupying a finite set of discrete states, such as cell types or epidemiological compartments \citep{pijuan2019single,kermack1927contribution}. A graph can then encode which state transitions are allowed. Discrete Wasserstein geometry provides a transport geometry on such graphs \citep{li2025geometric}. It introduces discrete continuity equations and corresponding kinetic energies for probability flows over the graph \citep{maas2011gradient,mielke2011gradient,chow2012fokker}. Within this framework, the minimum-energy motion associated with a population trajectory can be represented by a scalar potential defined on the graph \citep{chow2019discrete,li2022transport}. Related potential-based constructions have also been used to characterize kinetic-optimal discrete generative flows \citep{shaul2025kinetic}. These results provide the main geometric ingredients needed to formulate a finite-state analogue of Action Matching.

The main remaining challenge is to learn the action directly from samples. In continuous AM \citep{neklyudov2023action}, the kinetic term can be written as an expectation under the marginal distribution and can therefore be estimated from marginal samples. On a graph, however, the kinetic energy depends on an edge mobility determined by the probability masses at both endpoints of each edge. For nonlinear mobilities, these probabilities are unknown, so the kinetic term cannot be evaluated directly from a single marginal sample. A sample-based formulation must therefore handle this dependence on the unknown marginal probabilities while also learning the potential that drives the dynamics.

We address this problem with \emph{Discrete Action Matching} (DAM), a finite-state counterpart of Action Matching based on local density ratios. Density ratios can be estimated without evaluating the normalization constant of a discrete distribution \citep{hyvarinen2007some,dawid2012proper,kanamori2017graph}. Closely related ratios also play the role of discrete scores in discrete score-based generative models \citep{meng2022concrete,lou2024discrete}. Our key observation is that positive $1$-homogeneity allows the mobility on an edge to be factorized into the probability of one endpoint and a function of the density ratio between the two endpoints. The endpoint probability is then supplied implicitly by sampling from the marginal distribution, leaving only the local density ratio to estimate. This factorization makes the discrete action estimable from samples and separates density-ratio estimation from learning the action potential.

We assume that the graph and mobility function are fixed. Given a marginal path, DAM selects the unique current with minimum kinetic energy and uses it to construct a Markov process for trajectory generation. With finitely many observed marginals, we distinguish between two separate questions: how the marginal path is interpolated between observations, and how motion is selected along that path. We specify the interpolation used for training and evaluate it independently using held-out population measurements. DAM addresses the second question by selecting a canonical minimum-energy motion consistent with the prescribed marginal path. Since unpaired marginal observations generally do not determine a unique set of individual trajectories, this canonical motion should be interpreted as a principled reconstruction rather than as recovery of the exact historical paths \citep{lavenant2021towards,guan2024identifying}.

\textbf{Contributions.} Our main contributions are as follows:
\begin{itemize}[leftmargin=*, topsep=0.5mm, itemsep=0.5mm, parsep=0.5mm]
    \item \textbf{Discrete Action Matching.} In \wasyparagraph\ref{subsec:dam_derivation}, we derive a finite-state action-gap identity and show that minimizing the population objective recovers the canonical minimum-kinetic-energy current of a prescribed positive marginal path under a chosen graph transport geometry.

    \item \textbf{Learning from unpaired marginal samples.} We estimate neighboring density ratios with a local objective (\wasyparagraph\ref{subsec:dam_ratio_estimation}), then freeze the ratio model and learn the action potential (\wasyparagraph\ref{subsec:dam_training_sampling}). Both stages use samples from the same positive empirical interpolant.

    \item \textbf{Markov realization.} In \wasyparagraph\ref{subsec:dam_training_sampling}, we construct graph-supported nonnegative transition rates and establish the conditions under which the resulting Markov chain reproduces the prescribed marginal path. Appendix~\ref{app:dam_empirical_and_sampling} gives the derivation and the effect of ratio-estimation errors.
\end{itemize}

In \wasyparagraph\ref{sec:experiments} and Appendix~\ref{app:experiments}, we evaluate DAM on two tasks: reconstruction of observed marginals and interpolation between them. We consider both controlled synthetic dynamics and mouse gastrulation data. Compared with the graph-gradient-flow method of \citet{rancati2026learning}, DAM achieves lower mean marginal errors. We also find that the histogram-based variant of DAM performs comparably to the learned density-ratio variant. As an additional geometric evaluation, we consider transport on sphere and hand meshes, where DAM approximates numerical surface-transport trajectories \citep{lavenant2018surfaces} from paired samples (Appendix~\ref{app:mesh_protocol}).

The remainder of the paper is organized as follows. Section~\ref{sec:background} introduces the observation model, Action Matching, and discrete Wasserstein geometry. Section~\ref{sec:discrete_action_matching} develops DAM and its training and sampling procedures. Section~\ref{sec:related_work} discusses related methods, Section~\ref{sec:experiments} presents the main experiments, and Section~\ref{sec:discussion} discusses limitations and future directions. Detailed derivations and experimental protocols are provided in the corresponding Appendices.

\textbf{Notation.} We use $\Omega$ for a continuous state space and $I=\{1,\ldots,N\}$ for a finite state space. For $t\in[0,1]$, the corresponding marginals are denoted by $q(t)$ and $p(t)$, where $q(t)$ is a density and $p(t)$ is a probability vector with entries $p_i(t)$. We write $\mathcal P_2(\Omega)$ for probability measures with finite second moment and $\mathcal P_+(I)=\{p\in\mathbb R^N:p_i>0,\ \sum_i p_i=1\}$ for the positive probability simplex. The graph $G=(I,E,\omega)$ is connected and undirected, with edge set $E$ and symmetric positive edge weights $\omega_{ij}$. The scalar fields $s(t,\cdot)$ and $S(t,\cdot)$ are continuous and discrete action potentials, respectively. Their gradients $\nabla s$ and $\nabla_G S$ describe the associated continuous velocity and graph edge field; $\operatorname{div}_G$ is the graph divergence (\wasyparagraph\ref{subsec:discrete_wasserstein_background}). The notation $\theta(p)$ denotes edgewise mobility, with entries $\theta(p_i,p_j)$. Figure~\ref{fig:dam_pipeline} uses the shorthand $q_t=q(t)$, $p_t=p(t)$, $s_t=s(t,\cdot)$, and $S_t=S(t,\cdot)$.

\section{Background}
\label{sec:background}

We formulate learning population dynamics as an inverse problem from unpaired temporal marginals and review the geometric principle that selects canonical dynamics for a prescribed marginal path. We first introduce the observation model and distinguish the evolving population distribution from the unobserved trajectories of individual samples (\wasyparagraph\ref{subsec:population_dynamics_problem}). We then review continuous Action Matching (\wasyparagraph\ref{subsec:continuous_action_matching}) and the graph differential operators and discrete Wasserstein geometry needed to transfer this construction to finite state spaces (\wasyparagraph\ref{subsec:discrete_wasserstein_background}). Detailed derivations are provided in Appendix~\ref{sec:ot_continuous_and_discrete}.

\subsection{Learning dynamics from unpaired marginal snapshots}
\label{subsec:population_dynamics_problem}

Let $\gX$ denote the state space of an unknown stochastic process $(X_t)_{t\in[0,1]}$, and let $\rho(t):=\mathcal L(X_t)$ denote its one-time marginal law \citep{e2019applied}. We use $\gX=\Omega$ and write $q(t)$ when the state space is continuous, and use $\gX=I$ and write $p(t)$ when it is finite. The process induces a probability law on trajectories, whereas the collection $(\rho(t))_{t\in[0,1]}$ contains only its one-time marginals. At finitely many observation times $0=t_0<t_1<\cdots<t_K=1$, we observe independent samples
\begin{equation}
    \mathcal D_k:=\{x_k^{(m)}\}_{m=1}^{N_k},
    \qquad
    x_k^{(m)}\overset{\mathrm{i.i.d.}}{\sim}\rho(t_k),
    \label{eq:marginal_snapshot_data}
\end{equation}
with no correspondence between observations collected at different times. This observation model arises naturally when measurements are cross-sectional or destructive, as in temporal single-cell data \citep{schiebinger2019optimal,pijuan2019single}, and more generally whenever population snapshots are available but individual trajectories cannot be tracked \citep{tong2020trajectorynet}. Marginal observations do not determine how probability mass moves between states: different microscopic dynamics can generate the same marginal evolution, and finite snapshots additionally leave the intermediate marginal path unknown \citep{lavenant2021towards,guan2024identifying}. Population-dynamics inference therefore requires structural assumptions, such as path-space regularization \citep{chizat2022trajectory}, Wasserstein gradient flows \citep{bunne2022proximal,terpin2024learning,persiianov2026learning,guan2026gradientflow}, or action-based principles \citep{neklyudov2023action,neklyudov2024wasserstein,guan2026call}.

We seek the \emph{canonical dynamics} selected by a prescribed transport geometry and a sufficiently regular marginal path. With finite observations, this requires a choice of the path between snapshots; our training procedure uses the positive empirical interpolant \citep{albergo2025stochastic} defined in \eqref{eq:declared_interpolant}. The chosen path is part of the modeling assumptions.

A natural way to select canonical dynamics is to choose, among all motions producing the same infinitesimal change of the marginal, the one with minimum kinetic energy. This principle underlies optimal-transport geometry \citep{benamou2000computational,otto2001geometry,ambrosio2008gradient} and, for a prescribed marginal path, Action Matching \citep{neklyudov2023action}. We review this continuous construction next before transferring it to a finite state space.

\subsection{Action Matching on continuous sample spaces}
\label{subsec:continuous_action_matching}

Consider a smooth, uniformly positive path $q(t)$ of probability densities on a compact connected Riemannian domain $(\Omega,g^\Omega)$. The gradient, divergence, and volume element $dx$ are induced by the metric $g^\Omega$ \citep{lee2018introduction}. We use periodic or zero normal probability-flux boundary conditions and the regularity assumptions detailed in Appendix~\ref{subsec:continuous_optimal_transport}. A velocity field $v(t,\cdot)$ generates the marginal evolution through the \emph{continuity equation}
\begin{equation}
    \partial_t q(t)
    +\operatorname{div}\!\left(q(t)v(t)\right)=0,
    \label{eq:background_continuity_equation}
\end{equation}
which expresses conservation of probability mass \citep{benamou2000computational}. The marginal tangent $\partial_t q(t)$ does not determine $v(t)$ uniquely: adding a field with zero $q(t)$-weighted divergence leaves the marginal evolution unchanged. Wasserstein geometry selects the compatible velocity of smallest $L^2(q(t))$ norm \citep{ambrosio2008gradient,lott2008some}. In the smooth positive setting above, this velocity is a gradient field $\nabla s^\star(t)$, with $s^\star(t,\cdot)$ unique up to an additive constant:
\begin{equation}
    \nabla s^\star(t)
    \in
    \argmin_{\substack{v\\
        \partial_t q(t)=-\operatorname{div}(q(t)v)}}
    \frac12\int_\Omega
    g_x^\Omega\!\left(v(x),v(x)\right)q(t,x)\,dx.
    \label{eq:continuous_minimum_kinetic_velocity}
\end{equation}
In the general theory, Wasserstein tangent spaces are defined as the weighted $L^2$ closure of gradient vector fields \citep[Theorem~8.3.1 and Section~8.4]{ambrosio2008gradient}. Under our smoothness assumptions, this closure is not needed explicitly, and tangent velocities can be represented directly through smooth potentials. Appendix~\ref{subsec:continuous_optimal_transport} reviews how this minimum-energy potential representation relates to the weighted Hodge decomposition \citep{schwarz2006hodge} and to dynamical optimal transport.

Action Matching \citep{neklyudov2023action} learns $s^\star$ without observing either $\partial_t q(t)$ or trajectories. For a sufficiently regular candidate potential $s$, define the \emph{action gap}
\begin{equation}
    \operatorname{AG}(s,s^\star)
    :=
    \frac12\int_0^1
    \mathbb E_{x\sim q(t)}
    \left[
        \left\|\nabla s(t,x)-\nabla s^\star(t,x)\right\|_{g_x^\Omega}^2
    \right]dt,
    \label{eq:continuous_action_gap_background}
\end{equation}
where $\|u\|_{g_x^\Omega}^2:=g_x^\Omega(u,u)$. This discrepancy cannot be evaluated directly because $s^\star$ is unknown. Integration by parts in space and time removes the unknown target from the optimization. The Euclidean identity of \citet[Theorem~2.2]{neklyudov2023action} extends under the assumptions above to
$
    \operatorname{AG}(s,s^\star)
    =
    \mathcal L_{\mathrm{AM}}(s)
    -
    \mathcal L_{\mathrm{AM}}(s^\star),
$
where
\begin{align}
    \mathcal L_{\mathrm{AM}}(s)
    := {}&
    \mathbb E_{x\sim q(0)}[s(0,x)]
    -
    \mathbb E_{x\sim q(1)}[s(1,x)]
    \nonumber\\
    &+
    \int_0^1
    \mathbb E_{x\sim q(t)}
    \left[
        \partial_t s(t,x)
        +
        \frac12\|\nabla s(t,x)\|_{g_x^\Omega}^2
    \right]dt.
    \label{eq:continuous_am_background}
\end{align}
Thus, minimizing $\mathcal L_{\mathrm{AM}}$ over admissible potentials recovers the canonical gradient field. Since $q(t)$ enters the objective only through expectations, the loss can be estimated from marginal samples without evaluating the density, its time derivative, or the underlying velocity field \citep{neklyudov2023action}. However, this estimator requires samples from $q(t)$ at the times used to evaluate the time integral. When observations are available only at finitely many snapshots, the intermediate marginals are unknown and cannot be sampled directly.

For our development, the key feature of \eqref{eq:continuous_am_background} is its geometric structure. The construction requires a \emph{gradient} that maps a scalar potential to a local velocity and a density-dependent \emph{inner product} that measures kinetic energy. Discrete Wasserstein geometry provides graph counterparts of these objects \citep{maas2011gradient,chow2012fokker,mielke2011gradient}. We introduce them next.

\subsection{Graph Wasserstein geometry and canonical discrete dynamics}
\label{subsec:discrete_wasserstein_background}

We now specialize to a finite state space $I=\{1,\ldots,N\}$ equipped with a connected, undirected, weighted graph $G=(I,E,\omega)$ without self-loops \citep{li2016study}. For every edge $\{i,j\}\in E$, the weights are symmetric and positive, $\omega_{ij}=\omega_{ji}>0$, and $\mathcal N(i)$ denotes the set of neighbors of node $i$. The graph determines which pairs of states can exchange probability directly. We consider probability distributions in the positive simplex
$
\mathcal P_+(I)
:=
\left\{
p\in\mathbb R^N:
p_i>0,\ \sum_{i=1}^N p_i=1
\right\}.
$
Throughout, $E$ denotes the set of unordered edges, so each edge is counted once in sums over $E$. For a scalar potential $S\in\mathbb R^N$ and a skew-symmetric edge field $v$, with $v_{ij}=-v_{ji}$, define
\begin{align}
    (\nabla_G S)_{ij}
    &:=
    \sqrt{\omega_{ij}}(S_i-S_j),
    \label{eq:background_graph_gradient}\\
    (\operatorname{div}_G v)_i
    &:=
    \sum_{j\in\mathcal N(i)}
    \sqrt{\omega_{ij}}\,v_{ji}.
    \label{eq:background_graph_divergence}
\end{align}
With this sign convention, $-\operatorname{div}_G$ is the adjoint of $\nabla_G$ under the unweighted vertex and edge pairings \citep{chow2019discrete,li2022transport}; see Appendix~\ref{subsec:discrete_optimal_transport}.

The remaining ingredient is the counterpart of multiplication by the continuous density. Probability masses live on vertices, whereas velocities live on edges, so a graph transport geometry assigns an effective \emph{mobility} to each edge from the masses at its endpoints \citep{maas2011gradient,mielke2011gradient}. We use a symmetric mobility $\theta:[0,\infty)^2\to[0,\infty)$ that is smooth and strictly positive on $(0,\infty)^2$ and positively $1$-homogeneous:
$
    \theta(\lambda a,\lambda b)
    =
    \lambda\theta(a,b),
    \lambda>0.
$
The full adopted assumptions and the arithmetic, geometric, harmonic, and logarithmic examples are given in Definition~\ref{def:admissible_mean}. Different choices of $\theta$ generally induce different transport geometries. Write $\theta_{ij}(p):=\theta(p_i,p_j)$ and interpret $\theta(p)v$ edgewise. The mobility-weighted edge inner product and its norm are
\begin{equation*}
    \langle v,w\rangle_{p,\theta}
    :=
    \sum_{\{i,j\}\in E}\theta_{ij}(p)v_{ij}w_{ij},
    \qquad
    \|v\|_{p,\theta}^2:=\langle v,v\rangle_{p,\theta}.
\end{equation*}
The \emph{graph continuity equation} for a potential field is
\begin{equation}
    \dot p(t)
    =
    -\operatorname{div}_G
    \!\left(\theta(p(t))\nabla_G S(t)\right)
    =:
    L_\theta(p(t))S(t),
    \label{eq:background_discrete_continuity}
\end{equation}
or, componentwise,
\begin{equation}
    \dot p_i(t)
    =
    \sum_{j\in\mathcal N(i)}
    \omega_{ij}\theta(p_i(t),p_j(t))
    \bigl(S_i(t)-S_j(t)\bigr).
    \label{eq:background_discrete_continuity_coordinates}
\end{equation}
The corresponding kinetic energy is
\begin{equation}
    \mathcal K_{G,\theta}(p,S)
    :=
    \frac12\|\nabla_G S\|_{p,\theta}^2
    =
    \frac12\sum_{\{i,j\}\in E}
    \omega_{ij}\theta(p_i,p_j)(S_i-S_j)^2.
    \label{eq:background_discrete_kinetic_energy}
\end{equation}
These operators form the basic differential structure used in finite-state optimal transport \citep{maas2011gradient,mielke2011gradient,chow2012fokker}. In our notation, $p_i$ denotes the probability mass assigned to state $i$. Some reversible Markov-chain formulations instead use densities relative to a stationary distribution; Appendix~\ref{subsec:wasserstein_geometry_on_simplex} explains the correspondence between these two parameterizations.

The representation of marginal tangents by graph potentials is standard in discrete Wasserstein geometry \citep[Section~3]{maas2011gradient}. For a positive distribution $p$ on a connected graph, the mobility-weighted Laplacian $L_\theta(p)$ is positive semidefinite, its kernel is $\operatorname{span}\{\mathbf 1\}$, and its range is the tangent space
$
    T_p\mathcal P_+(I)
    =
    \{\sigma\in\mathbb R^N:\mathbf 1^\top\sigma=0\}.
$
Hence every tangent vector $\sigma\in T_p\mathcal P_+(I)$ can be represented by a potential $S^\star$, unique up to an additive constant, such that
\begin{equation}
    \sigma
    =
    L_\theta(p)S^\star
    =
    -\operatorname{div}_G
    \!\left(\theta(p)\nabla_G S^\star\right).
    \label{eq:canonical_discrete_tangent}
\end{equation}
Among all edge fields that realize the same tangent $\sigma$, the gradient field $\nabla_G S^\star$ has minimum kinetic energy \citep[Lemma~3.6]{maas2011gradient}; see also \citet[Lemma~4]{chow2019discrete}:
\begin{equation}
    \nabla_G S^\star
    \in
    \argmin_{\substack{v_{ij}=-v_{ji}\\
        \sigma=-\operatorname{div}_G(\theta(p)v)}}
    \frac12\|v\|_{p,\theta}^2.
    \label{eq:canonical_discrete_minimum_energy}
\end{equation}
Appendix~\ref{subsec:wasserstein_geometry_on_simplex} reviews this construction in our notation, including the graph-Laplacian representation and the continuous--discrete operator dictionary in Table~\ref{tab:operator_dictionary}. Therefore, along any smooth positive probability path $p(t)$, there is a canonical potential $S^\star(t)$, unique up to a time-dependent additive constant, satisfying
\begin{equation}
    \dot p(t)
    =
    L_\theta(p(t))S^\star(t).
    \label{eq:canonical_discrete_path}
\end{equation}

This provides the geometric basis for a finite-state counterpart of Action Matching. The main additional challenge is statistical: the discrete kinetic energy depends on the unknown endpoint masses through $\theta(p_i(t),p_j(t))$. For nonlinear mobilities, this prevents an edge contribution from being evaluated directly from a single marginal sample. In \wasyparagraph\ref{sec:discrete_action_matching}, we use positive $1$-homogeneity to rewrite this dependence in terms of neighboring density ratios, leading to a sample-based action objective.

\section{Method}
\label{sec:discrete_action_matching}

We first derive the discrete action objective (\wasyparagraph\ref{subsec:dam_derivation}) and describe density-ratio estimation along an empirical interpolant (\wasyparagraph\ref{subsec:dam_ratio_estimation}). We then learn the action potential with the ratio model held fixed and construct a Markov realization of the dynamics (\wasyparagraph\ref{subsec:dam_training_sampling}).

\subsection{Discrete Action Matching}
\label{subsec:dam_derivation}

Let $p\in C^2([0,1];\mathcal P_+(I))$ be a prescribed probability path on a fixed connected graph $G$, and let the symmetric mobility be smooth and positive in the interior of the simplex. By \eqref{eq:canonical_discrete_tangent}, each marginal velocity $\dot p(t)$ admits a canonical potential $S^\star(t)\in\mathbb R^N$, unique up to a time-dependent additive constant, such that
\eqref{eq:canonical_discrete_path} holds. Equivalently, $\nabla_G S^\star(t)$ is the minimum-kinetic-energy edge field that realizes the marginal velocity $\dot p(t)$. Our goal is to recover this canonical potential from marginal samples, without observing either the marginal velocities $\dot p(t)$ or individual trajectories. Following continuous Action Matching \citep{neklyudov2023action}, for any candidate potential $S\in C^1([0,1];\mathbb R^N)$ we measure its discrepancy from $S^\star$ by the \emph{discrete action gap}
\begin{equation}
    \operatorname{AG}_{G,\theta}(S,S^\star)
    :=
    \frac12\int_0^1
    \left\|
        \nabla_G S(t)-\nabla_G S^\star(t)
    \right\|_{p(t),\theta}^2\,dt,
    \label{eq:dam_action_gap}
\end{equation}
where $\|\cdot\|_{p(t),\theta}$ is the mobility-weighted edge norm defined in \wasyparagraph\ref{sec:background}. This objective is not directly usable for learning because the target potential $S^\star$ is unknown. However, graph summation by parts and integration by parts in time allow us to rewrite the action gap, up to terms independent of $S$, without explicit dependence on $S^\star$.

\begin{theorem}[Discrete Action Matching]
    \label{thm:discrete_action_matching}
    Fix a connected undirected graph with symmetric positive edge weights and a symmetric mobility $\theta$ that is smooth and strictly positive on $(0,\infty)^2$. Let $p\in C^2([0,1];\mathcal P_+(I))$ and choose
    $
        S^\star(t)
        =
        L_\theta(p(t))^\dagger\dot p(t),
    $
    where $\dagger$ denotes the Moore--Penrose pseudoinverse \citep{moore1920reciprocal, penrose1955generalized}. This is the zero-mean canonical potential; the constant rank of $L_\theta(p(t))$ gives $S^\star\in C^1([0,1];\mathbb R^N)$. For $S\in C^1([0,1];\mathbb R^N)$, define
    \begin{align}
        \mathcal L_{\mathrm{DAM}}(S)
        := {}&
        \mathbb E_{i\sim p(0)}[S(0,i)]
        -
        \mathbb E_{i\sim p(1)}[S(1,i)]
        \nonumber\\
        &+
        \int_0^1
        \left\{
            \mathbb E_{i\sim p(t)}[\partial_t S(t,i)]
            +
            \mathcal K_{G,\theta}(p(t),S(t))
        \right\}dt,
        \label{eq:dam_population_loss}
    \end{align}
    where the kinetic energy $\mathcal K_{G,\theta}(p,S)$ defined at \eqref{eq:background_discrete_kinetic_energy}. Then
    $
        \mathcal L_{\mathrm{DAM}}(S)
        -
        \mathcal L_{\mathrm{DAM}}(S^\star)
        =
        \operatorname{AG}_{G,\theta}(S,S^\star).
    $
    Consequently, the minimizers over $C^1([0,1];\mathbb R^N)$ are precisely
    $
        S(t)=S^\star(t)+c(t)\mathbf1,
        \, c\in C^1([0,1]),
    $
    with equality at every $t\in[0,1]$.
\end{theorem}

The proof in Appendix~\ref{subsec:wasserstein_geometry_on_simplex} is the graph counterpart of the continuous Action Matching identity of \citet[Theorem~2.2]{neklyudov2023action}. For a fixed probability path and graph geometry, the theorem characterizes the population-level optimum over all admissible potentials. In practice, a sufficiently expressive model class and finite-sample optimization provide an approximation to this target.

If the probabilities $p(t)$ are known, the population loss in \eqref{eq:dam_population_loss} can be evaluated directly. When only marginal samples are available, the boundary and time-derivative terms remain directly estimable because they are expectations under the corresponding marginals. The kinetic term is more difficult: for nonlinear mobilities, each edge contribution depends on $\theta(p_i(t),p_j(t))$, and hence on the unknown probability masses at both endpoints.

For sample-based evaluation, assume additionally that $\theta$ is positively $1$-homogeneous. With $\kappa_\theta(u):=\theta(1,e^u)$, homogeneity gives $\theta(p_i,p_j)=p_i\kappa_\theta(\log(p_j/p_i))$, hence
\begin{equation}
    \mathcal K_{G,\theta}(p,S)
    =\frac14\mathbb E_{i\sim p}
    \left[\sum_{j\in\mathcal N(i)}
        \omega_{ij}\kappa_\theta\!\left(\log\frac{p_j}{p_i}\right)
        (S_i-S_j)^2\right].
    \label{eq:dam_sample_local_kinetic}
\end{equation}
Thus local ratio estimates suffice to evaluate the action without reconstructing the full probability vector.\footnote{For arithmetic mobility, the kinetic action admits a ratio-free sample estimator. Our Markov realization still uses density ratios; see Appendix~\ref{app:ratios_in_dam}.} Appendix~\ref{app:ratios_in_dam} derives this representation, including the factor $1/4$.

\subsection{Learning local density ratios}
\label{subsec:dam_ratio_estimation}

We estimate neighboring ratios through a scalar log-density field, following normalization-free estimation on discrete spaces \citep{hyvarinen2007some,kanamori2017graph}. The field $U_\phi(t,i)$ targets $U^\star(t,i)=\log p_i(t)+c(t)$. Representing ratios through a single node potential automatically enforces reciprocity and cycle consistency (Appendix~\ref{app:discrete_scores}). We parameterize the action potential $S_\eta$ separately, so density-ratio estimation and learning the population dynamics remain distinct.

To obtain samples at intermediate times from finitely many snapshots, let $\widehat p_k$ denote the empirical distribution of $\mathcal D_k$, let $r$ be the uniform reference distribution with $r_i=1/N$, and fix $0<\delta<1$. For $t\in[t_k,t_{k+1}]$, define $u=(t-t_k)/(t_{k+1}-t_k)$ and $h(u)=3u^2-2u^3$ and set
\begin{equation}
    \widetilde p(t)
    =
    (1-\delta)
    \bigl[(1-h(u))\widehat p_k+h(u)\widehat p_{k+1}\bigr]
    +\delta r.
    \label{eq:declared_interpolant}
\end{equation}
Samples from $\widetilde p(t)$ can be generated directly from the two neighboring snapshots, without pairing observations across them. The reference mixture also guarantees full support, $\widetilde p_i(t)\geq\delta/N$. Both stages of DAM are trained on this same interpolated path, including the reference-mixed snapshot distributions at its endpoints. The interpolation and reference mixing therefore provide a well-defined population path for training from finite snapshots, while leaving the unknown intermediate population evolution to be assessed empirically. The resulting path is $C^1$ and piecewise smooth, and the action identity applies interval by interval as detailed in Appendix~\ref{app:dam_empirical_and_sampling}.

Using fresh samples $T_b\sim\operatorname{Unif}[0,1]$ and $I_b\sim\widetilde p(T_b)$, we minimize
\begin{equation}
    \widehat{\mathcal R}(\phi)
    :=
    \frac1B\sum_{b=1}^B\sum_{j\in\mathcal N(I_b)}
    \omega_{I_bj}\operatorname{softplus}
    \!\left(U_\phi(T_b,j)-U_\phi(T_b,I_b)\right),
    \label{eq:dam_empirical_ratio_loss}
\end{equation}
where $\operatorname{softplus}(z):=\log(1+e^z)$. For any positive marginal $p$, the population minimizers are exactly $U=\log p+c\mathbf 1$ (Appendix~\ref{app:ratios_in_dam}, \eqref{eq:softplus_ratio_argmin}). Therefore, if the model class can represent these minimizers for all times and the population objective is minimized exactly, the time-integrated loss recovers the neighboring density ratios of $\widetilde p(t)$ for almost every $t$. If the fitted ratios are continuous in time, this equality extends to all $t\in[0,1]$. This result characterizes the population target associated with the chosen interpolated path; finite-sample estimation of that target is studied empirically.

\subsection{Training and Markov realization}
\label{subsec:dam_training_sampling}

After fitting $U_\phi$, we freeze $\phi$ and optimize $S_\eta$, using the estimated local density ratios to evaluate the prescribed mobility from samples. Define
\begin{equation}
    \widehat{\mathcal K}_{\eta,\phi}(t,i)
    :=\frac14\sum_{j\in\mathcal N(i)}
    \omega_{ij}\kappa_\theta\!\left(U_\phi(t,j)-U_\phi(t,i)\right)
    \bigl(S_\eta(t,j)-S_\eta(t,i)\bigr)^2.
    \label{eq:dam_empirical_local_kinetic}
\end{equation}
With fresh $(T_b,I_b)$ and independent endpoint samples $I_b^0\sim\widetilde p(0)$ and $I_b^1\sim\widetilde p(1)$, minimize
\begin{equation}
    \widehat{\mathcal L}_{\mathrm{DAM}}(\eta;\phi)
    :=\frac1B\sum_{b=1}^B
    \Bigl[S_\eta(0,I_b^0)-S_\eta(1,I_b^1)
        +\partial_t S_\eta(T_b,I_b)
        +\widehat{\mathcal K}_{\eta,\phi}(T_b,I_b)\Bigr].
    \label{eq:dam_empirical_action_loss}
\end{equation}
Time derivatives use automatic differentiation in PyTorch \citep{paszke2019pytorch}. Conditional on the data and frozen $U_\phi$, this estimator is unbiased at fixed $S_\eta$ for the action with the effective positive symmetric conductances in \eqref{eq:perturbed_conductance}. Exact ratios recover the intended-mobility objective. Appendix~\ref{app:dam_empirical_and_sampling} gives regularity and estimation details.

For sampling, let $\psi_\varepsilon(z):=\varepsilon\log(1+e^{z/\varepsilon})$, $\varepsilon>0$, so $\psi_\varepsilon(z)-\psi_\varepsilon(-z)=z$. For neighboring states $i\ne j$, set
\begin{equation}
    Q_t^{\eta,\phi}(i,j)
    :=\omega_{ij}\kappa_\theta\!\left(U_\phi(t,j)-U_\phi(t,i)\right)
    \psi_\varepsilon\!\left(S_\eta(t,j)-S_\eta(t,i)\right).
    \label{eq:dam_generator_offdiag}
\end{equation}
Other off-diagonal entries vanish and $Q_t^{\eta,\phi}(i,i):=-\sum_{j\ne i}Q_t^{\eta,\phi}(i,j)$. Finite continuous fields give a bounded graph-supported Markov generator \citep{anderson2012continuous}. Its column-vector marginal law satisfies $\dot\mu=Q_t^\top\mu$ \citep{pauline2025foundations}. At exact ratios for a positive target path $p$, $Q_t^\top p=L_\theta(p)S_\eta$. If additionally $S_\eta$ is canonical and $\mu(0)=p(0)$, uniqueness gives $\mu(t)=p(t)$ throughout. The split selects one of multiple compatible Markov realizations.

Imperfect ratios preserve generator validity but can separate its current from the effective training current, even at zero potential difference (Appendix~\ref{app:dam_empirical_and_sampling}). For finite snapshots, ideal reproduction concerns the declared $\widetilde p$; approximate fits generally yield a different law $\mu$.

\section{Related Work}
\label{sec:related_work}

\paragraph{Action Matching.} Action Matching learns continuous population dynamics from independent temporal marginals through a sample-based variational objective \citep{neklyudov2023action}. Subsequent work extends this principle to Wasserstein Lagrangian Flows \citep{neklyudov2024wasserstein}, parameter-dependent dynamics \citep{berman2024parametric}, and more general Wasserstein mechanics \citep{guan2026call}. DAM transfers the fixed-path Action Matching principle to finite state spaces using mobility-dependent discrete Wasserstein geometry.

\paragraph{Discrete Wasserstein geometry and gradient-flow learning.} Discrete transport geometries on graphs provide finite-state analogues of Wasserstein continuity equations, kinetic energies, and gradient flows \citep{maas2011gradient,chow2012fokker,erbar2014gradient}. These constructions have subsequently been developed as Riemannian geometries on the probability simplex and used to study discrete gradient-flow dynamics \citep{chow2018entropy,li2018natural,li2022transport}. Most closely related to our work, \citet{rancati2026learning} learn finite-state population dynamics by assuming an autonomous free-energy gradient flow and recovering its energy through discrete-JKO optimality conditions, analogous to the continuous inverse-JKO approach of \citet{terpin2024learning}. Their method requires numerical computation of discrete-Wasserstein geodesics between observed marginals. DAM addresses a complementary inverse problem: it learns the canonical, potentially time-dependent velocity of a prescribed marginal evolution and uses local density ratios to make the corresponding action directly estimable from samples without discrete-geodesic preprocessing.

Further connections to discrete transport, inverse gradient-flow learning, and discrete generative models are discussed in Appendix~\ref{app:extended_related_work}.

\section{Experiments}
\label{sec:experiments}

\begin{figure}[tbp]
    \centering
    \includegraphics[width=\linewidth]{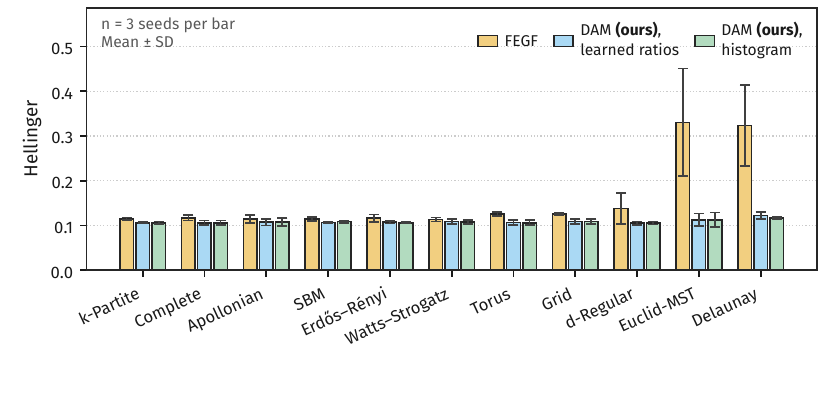}
    \vspace{-13mm}
    \caption{Synthetic reconstruction across eleven graph families with $N=60$, $M=1{,}000$ samples per split and time, and $\beta=0.1$. DAM uses logarithmic mobility; lower is better.}
    \vspace{-5mm}
    \label{fig:rancati_topology}
\end{figure}

We adapt the synthetic and mouse-gastrulation benchmarks of \citet{rancati2026learning} and compare DAM with their free-energy gradient-flow method (FEGF), starting from the official implementation.\footnote{\url{https://github.com/dariorancati/discrete-gradient-flow}} Methods receive matched sampled inputs, with validation restricted to retained observation times. Learned-ratio and histogram DAM share the positive empirical interpolant and action-training settings. Figures report mean Hellinger errors and sample standard deviations across three seeds. Protocol adaptations, including FEGF's time-indexed gastrulation potential and the methods' different transport geometries, are detailed in Appendix~\ref{app:implementation}.

\subsection{Synthetic graph dynamics}
\label{sec:synthetic_experiment}

\paragraph{Setup.} We retain five unpaired marginal snapshots from synthetic free-energy dynamics on eleven graph families. Each model is initialized from the empirical training distribution at the first snapshot and then propagated continuously to all subsequent observation times, without reinitialization at intermediate snapshots. We compare the predicted marginals with independent test histograms at the remaining four observation times and average the resulting errors. This evaluates how well each method reconstructs the observed marginals from a single forward rollout. 

\paragraph{Discussion.} With logarithmic mobility, both DAM variants achieve lower mean marginal errors than FEGF on all eleven graph families in the base experimental setting (Figure~\ref{fig:rancati_topology}). The advantage largely persists when varying entropy strength, graph size, and sample budget. The histogram-based and learned-ratio variants of DAM achieve comparable errors when results are pooled across these settings (Appendix~\ref{app:fegf_ablations}).

\subsection{Mouse gastrulation}
\label{sec:gastrulation_experiment}

\paragraph{Setup.} The gastrulation benchmark contains $34$ cell types observed at nine developmental stages from the atlas of \citet{pijuan2019single}. We construct a lineage-weighted graph with additional weak background connections. To evaluate interpolation, we leave out each interior stage in turn and refit the model using the remaining eight stages, including observations both before and after the omitted stage. Starting from the preceding retained empirical marginal, we then propagate the model to the omitted time. The held-out, pseudocount-smoothed atlas marginal is used only for evaluation and is not included in training. 

\paragraph{Discussion.} Both DAM variants achieve lower mean interpolation error than FEGF (Figure~\ref{fig:gastrulation_loo}). Among the four mobility choices considered, harmonic mobility gives the lowest mean errors for DAM. Full results are reported in Table~\ref{tab:mobility_closure}. Variation across seeds reflects resampling and optimization on the same underlying atlas. The learned-ratio variant performs comparably to histogram-based DAM, supporting density-ratio estimation as a more scalable alternative when explicit marginal histograms become difficult to estimate. Appendix~\ref{app:experiments} provides the full experimental protocol and supplementary mesh experiments.

\begin{figure}[tbp]
    \centering
    \includegraphics[width=\linewidth]{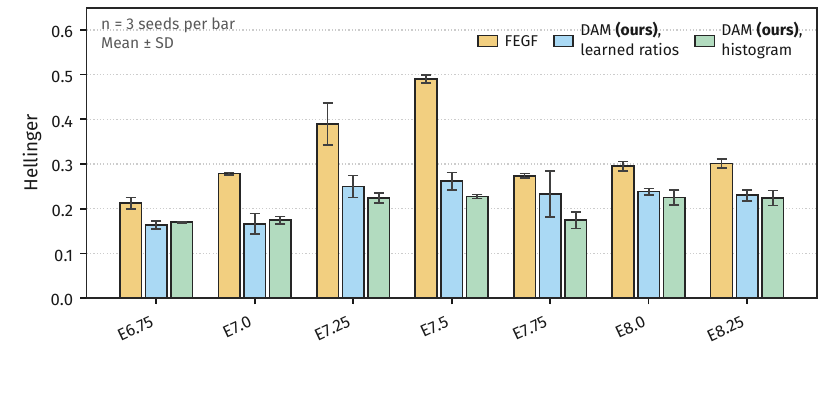}
    \vspace{-14mm}
    \caption{Gastrulation interpolation against omitted smoothed atlas marginals. DAM uses harmonic mobility; lower is better.}
    \vspace{-4mm}
    \label{fig:gastrulation_loo}
\end{figure}

\section{Discussion \& Future Work}
\label{sec:discussion}

\paragraph{Limitations.}
Temporal marginals generally do not determine a unique microscopic history \citep{lavenant2021towards,guan2024identifying}. DAM therefore selects a principled canonical solution: the minimum-action current determined by the prescribed graph geometry and mobility. At the trajectory level, this signed current can admit multiple compatible Markov realizations. Our smooth positive split provides one such realization with the correct marginal evolution, while alternative decompositions offer additional flexibility in how microscopic paths are represented.

\paragraph{Toward discrete Wasserstein Lagrangian Flows.}
A natural extension is a discrete counterpart of Wasserstein Lagrangian Flows \citep{neklyudov2024wasserstein}. While Action Matching selects the canonical velocity of a prescribed path \citep{neklyudov2023action}, Wasserstein Lagrangian Flows additionally optimize the interpolation between finitely observed marginals. Combining this principle with DAM would require a parameterized marginal path and an outer objective enforcing observation constraints, replacing our predefined empirical interpolant with a variationally selected path. Conditional and unbalanced \citep{chizat2017unbalanced} extensions for cortical biomarkers \citep{huang2024cortical} with changing total burden provide another future direction (see extended discussion in Appendix~\ref{app:cortical_prediction}).

\paragraph{Relation to free-energy models.} DAM complements methods such as FEGF \citep{rancati2026learning}, which assume that the evolution is generated by a single autonomous free-energy functional. DAM instead allows a time-dependent potential and therefore applies to more general prescribed marginal evolutions. We hope that proposed local density-ratio estimation provides a route for transferring further Wasserstein variational methods from continuous to discrete population dynamics.

\section*{AI use statement}

Generative AI tools, including Codex and Claude, assisted with literature searches, code development and debugging, mathematical derivations, and manuscript drafting and editing. The authors reviewed and verified the AI-assisted content and take full responsibility for the theoretical claims, implementation, experimental results, and final manuscript.

\section*{Ethics statement}

This work develops methods for learning population dynamics from temporal snapshots on finite state spaces. The reported experiments use synthetic data, existing mouse gastrulation atlas data, and samples generated from numerical transport on surface meshes. No human participants were recruited and no new biological data were collected for this study. The inferred dynamics are model-dependent descriptions of population change and do not establish biological causation or clinical validity. Cortical prediction is discussed as a future application.

\section*{Reproducibility statement}

The geometric assumptions and theoretical proofs are provided in Appendix~\ref{sec:ot_continuous_and_discrete}. Appendices~\ref{app:dam_empirical_and_sampling} and~\ref{app:experiments} document the training objectives, empirical interpolation, support treatment, dataset preprocessing, baseline adaptations, configurations, evaluation metrics, and numerical checks.

\bibliography{iclr2027_conference}
\bibliographystyle{iclr2027_conference}

\newpage
\appendix

\section*{Contents}
\vspace{-5mm}
\startcontents[sections]
\printcontents[sections]{l}{1}{\setcounter{tocdepth}{2}}

\section{Continuous and discrete Wasserstein geometry}
\label{sec:ot_continuous_and_discrete}

This section provides the geometric foundations used in Sections~\ref{sec:background} and~\ref{sec:discrete_action_matching}. We follow the dynamical viewpoint on optimal transport, in which transport is formulated as a minimum-action problem over probability paths \citep{benamou2000computational,otto2001geometry,ambrosio2008gradient}. In continuous spaces, this viewpoint yields the formal Riemannian structure of the $2$-Wasserstein space associated with Otto calculus \citep{otto2001geometry,lott2008some}. On finite state spaces, graph gradients, mobility-weighted continuity equations, and kinetic energies give an analogous geometry on the relative interior of the probability simplex \citep{maas2011gradient,chow2012fokker,erbar2014gradient}. Related finite-dimensional differential-geometric formulations are developed by \citet{li2018natural,li2022transport,li2025geometric}.

Our presentation is organized around the distinction between a \emph{global transport problem}, which optimizes the probability path itself, and a \emph{fixed-path tangent problem}, which selects a canonical velocity for a prescribed marginal evolution. The latter is the geometric principle underlying Action Matching \citep{neklyudov2023action}. We first review this distinction in continuous spaces. We then introduce its graph counterpart for a general admissible mobility $\theta$, and finally make the induced Riemannian structure on the probability simplex explicit through a mobility-weighted graph Laplacian. The latter representation also gives the existence, uniqueness, and minimum-action characterization of the canonical discrete tangent used in the main text.

For comprehensive treatments of optimal transport, we refer to \citet{villani2009optimal,santambrogio2015optimal,ambrosio2008gradient}; computational aspects are reviewed by \citet{gabriel2019computational}. Accessible introductions include \citet{figalli2021invitation,peyre2025optimal}. Our continuous calculations use standard notions from Riemannian geometry; see, e.g., \citet{petersen2006riemannian,michor2008topics,lee2018introduction}, or \citet{hauberg2021diffgeom} for a machine-learning-oriented introduction.

\subsection{Optimal transport and Action Matching on continuous sample spaces}
\label{subsec:continuous_optimal_transport}

Let $(\Omega,g^\Omega)$ be a finite-dimensional smooth Riemannian manifold with volume form $dx$. For every $x\in\Omega$, the Riemannian metric
$
    g_x^\Omega:
    T_x\Omega\times T_x\Omega
    \longrightarrow
    \mathbb R
$
is an inner product on the tangent space $T_x\Omega$; we write
\begin{equation*}
    \Vert v\Vert_{g_x^\Omega}^2
    :=
    g_x^\Omega(v,v).
\end{equation*}
We denote by $\mathcal P_2(\Omega)$ the Borel probability measures on $\Omega$ with finite second moment. Whenever a measure is absolutely continuous with respect to $dx$, we identify it with its density. For the displayed differential calculations, take a compact connected smooth Riemannian domain, $q\in C^1([0,1];C^2(\Omega))$ uniformly positive, smooth spatial vector fields, and candidates $s\in C^1([0,1];C^2(\Omega))$. Use periodic conditions or zero normal probability flux; assume the canonical potential has the same regularity whenever its derivatives are invoked. These are sufficient formal assumptions, not a general theorem for every absolutely continuous Wasserstein curve.

\paragraph{Benamou--Brenier and potential velocities.}
Given two endpoint distributions $q^0,q^1\in\mathcal P_2(\Omega)$, the dynamical formulation of optimal transport writes the quadratic Wasserstein distance as
\begin{equation*}
    \frac12
    W_2^2(q^0,q^1)
    =
    \inf_{(q_t,v_t)}
    \frac12
    \int_0^1
    \int_\Omega
    \Vert v_t(x)\Vert_{g_x^\Omega}^2
    q_t(x)\,dx\,dt,
\end{equation*}
where the infimum is over probability paths $t\mapsto q_t$ and velocity fields $v_t$ satisfying the continuity equation
\begin{equation}
    \partial_t q_t
    +
    \operatorname{div}(q_tv_t)
    =
    0,
    \qquad
    q_0=q^0,
    \qquad
    q_1=q^1.
    \label{eq:continuous_continuity}
\end{equation}
This formulation originates with \citet{benamou2000computational}; its relation to absolutely continuous curves in Wasserstein space is developed systematically in \citet{ambrosio2008gradient}. Equation~\eqref{eq:continuous_continuity} expresses conservation of probability mass under the velocity $v_t$; see also \citet[Lemma~4.1.1]{figalli2021invitation}.

The potential form of the optimal velocity can be seen directly from the constrained problem, as reviewed in \citet[Appendix~B]{neklyudov2023action}. Introduce a Lagrange multiplier $s_t:\Omega\to\mathbb R$ for \eqref{eq:continuous_continuity}. Ignoring endpoint terms that do not depend on $v$, the velocity-dependent part of the augmented functional is
\begin{equation*}
    \int_0^1
    \int_\Omega
    \left[
        \frac12
        \Vert v_t\Vert_{g^\Omega}^2
        -
        g_x^\Omega(\nabla s_t,v_t)
        \right]
    q_t\,dx\,dt,
\end{equation*}
where we used the Riemannian integration-by-parts identity
\begin{equation*}
    \int_\Omega
    s\,\operatorname{div}(qv)\,dx
    =
    -
    \int_\Omega
    q\,g_x^\Omega(\nabla s,v)\,dx
\end{equation*}
under the stated boundary assumptions \citep{lee2018introduction}. Completing the square,
\begin{equation*}
    \frac12\Vert v\Vert_{g^\Omega}^2
    -
    g^\Omega(\nabla s,v)
    =
    \frac12
    \Vert v-\nabla s\Vert_{g^\Omega}^2
    -
    \frac12
    \Vert\nabla s\Vert_{g^\Omega}^2,
\end{equation*}
shows that the minimizing velocity is
$
    v_t=\nabla s_t.
$
The corresponding continuity equation is therefore
\begin{equation*}
    \partial_tq_t
    =
    -
    \operatorname{div}
    \left(
    q_t\nabla s_t
    \right).
\end{equation*}
Only differences of the potential matter, so $s_t$ is defined up to an additive time-dependent constant.

\paragraph{The fixed-path tangent problem.}
The Benamou--Brenier problem simultaneously chooses a probability path and a velocity field. Action Matching instead starts from a prescribed marginal path and asks which velocity should represent its instantaneous motion \citep{neklyudov2023action}. This distinction is most transparent through the local Wasserstein tangent problem.

Formally, consider the space of smooth strictly positive densities
$
    \mathcal P_+(\Omega)
    :=
    \left\{
    q\in C^\infty(\Omega):
    q>0,\quad
    \int_\Omega q\,dx=1
    \right\}.
$
Its tangent space consists of zero-mass perturbations,
$
    T_q\mathcal P_+(\Omega)
    =
    \left\{
    \sigma:
    \int_\Omega\sigma(x)\,dx=0
    \right\}.
$
This density-manifold viewpoint is classical in information geometry and optimal transport \citep{lafferty1988density,otto2001geometry,lott2008some}.

For a prescribed tangent $\sigma\in T_q\mathcal P_+(\Omega)$, the equation
$
    \sigma
    =
    -
    \operatorname{div}(qv)
$
does not determine an arbitrary velocity $v$ uniquely. If $w$ satisfies
$
    \operatorname{div}(qw)=0,
$
then $v+w$ generates the same tangent. In the weighted $L^2(q)$ geometry, this ambiguity is resolved by the weighted Helmholtz--Hodge decomposition: the gradient fields form the orthogonal complement of the $q$-divergence-free fields \citep[Lemma~8.4.2]{ambrosio2008gradient}; see also \citet{schwarz2006hodge,chow2018entropy}. Formally,
\begin{equation*}
    v
    =
    \nabla s+w,
    \qquad
    \operatorname{div}(qw)=0,
\end{equation*}
with
$
    \int_\Omega
    q\,
    g_x^\Omega(\nabla s,w)\,dx
    =
    0.
$
Consequently,
\begin{equation}
    \frac12
    \int_\Omega
    q\,\Vert v\Vert_{g^\Omega}^2\,dx
    =
    \frac12
    \int_\Omega
    q\,\Vert\nabla s\Vert_{g^\Omega}^2\,dx
    +
    \frac12
    \int_\Omega
    q\,\Vert w\Vert_{g^\Omega}^2\,dx.
    \label{eq:continuous_hodge_energy}
\end{equation}
The gradient component is therefore the unique minimum-norm representative of the tangent. This motivates the Wasserstein kinetic energy of $\sigma$ at $q$,
\begin{equation*}
    \mathcal K_W(q,\sigma)
    :=
    \inf_{\substack{v\\
            \sigma=-\operatorname{div}(qv)}}
    \frac12
    \int_\Omega
    q(x)
    \Vert v(x)\Vert_{g_x^\Omega}^2\,dx.
\end{equation*}
For absolutely continuous Wasserstein curves, the existence and minimum-norm characterization of the canonical velocity are made rigorous in \citet[Theorem~8.3.1 and Section~8.4]{ambrosio2008gradient}. The minimizer is the gradient representative
\begin{equation*}
    \sigma
    =
    -
    \operatorname{div}
    \left(
    q\nabla s
    \right),
\end{equation*}
and
\begin{equation*}
    \mathcal K_W(q,\sigma)
    =
    \frac12
    \int_\Omega
    q(x)
    \Vert\nabla s(x)\Vert_{g_x^\Omega}^2\,dx.
\end{equation*}
Along an absolutely continuous Wasserstein curve, the same canonical velocity also arises infinitesimally from optimal transport between nearby marginals \citep[Proposition~8.4.6]{ambrosio2008gradient}; see also \citet[Appendix~B.1]{neklyudov2023action}.

\paragraph{Wasserstein metric and path action.}
The map
\begin{equation*}
    C^\infty(\Omega)/\mathbb R
    \longrightarrow
    T_q\mathcal P_+(\Omega),
    \qquad
    [s]
    \longmapsto
    -\operatorname{div}(q\nabla s),
\end{equation*}
identifies potentials modulo constants with tangent directions under suitable regularity assumptions \citep{otto2001geometry,lott2008some}. If
\begin{equation*}
    \sigma_s=-\operatorname{div}(q\nabla s),
    \qquad
    \sigma_r=-\operatorname{div}(q\nabla r),
\end{equation*}
the formal Wasserstein inner product is
\begin{equation*}
    g_q^W(\sigma_s,\sigma_r)
    :=
    \int_\Omega
    q(x)\,
    g_x^\Omega
    \left(
    \nabla s(x),
    \nabla r(x)
    \right)
    dx.
\end{equation*}
Thus
$
    \mathcal K_W(q,\sigma)
    =
    \frac12
    g_q^W(\sigma,\sigma).
$
The family $g^W=\{g_q^W\}_q$ is the formal Riemannian structure underlying Otto calculus \citep{otto2001geometry,lott2008some}; rigorous formulations use Wasserstein tangent spaces, metric derivatives, and absolutely continuous curves \citep{ambrosio2008gradient}.

For a smooth path $q_t$ with canonical potential $s_t$,
\begin{equation*}
    \dot q_t
    =
    -
    \operatorname{div}(q_t\nabla s_t),
\end{equation*}
the corresponding kinetic action is
\begin{equation*}
    \mathcal A_W[q]
    :=
    \int_0^1
    \mathcal K_W(q_t,\dot q_t)\,dt
    =
    \frac12
    \int_0^1
    \int_\Omega
    q_t(x)
    \Vert\nabla s_t(x)\Vert_{g_x^\Omega}^2
    dx\,dt.
\end{equation*}
The Benamou--Brenier problem can therefore be written as
\begin{equation}
    \frac12
    W_2^2(q^0,q^1)
    =
    \inf_{\substack{q_0=q^0\\q_1=q^1}}
    \mathcal A_W[q].
    \label{eq:bb_as_action}
\end{equation}
The local tangent problem selects the canonical velocity of a \emph{given} path, whereas \eqref{eq:bb_as_action} additionally minimizes over the path connecting the endpoints.

\paragraph{Fixed-path Action Matching.}
Action Matching \citep{neklyudov2023action} keeps the complete marginal path $q_t$ fixed and learns its canonical potential $s_t^\star$. By \citet[Theorem~2.1]{neklyudov2023action}, for an absolutely continuous Wasserstein path there exists a unique canonical gradient field $\nabla s_t^\star$ satisfying
\begin{equation}
    \dot q_t
    =
    -
    \operatorname{div}
    \left(
    q_t\nabla s_t^\star
    \right),
    \label{eq:continuous_am_target}
\end{equation}
with $s_t^\star$ unique up to an additive function of time.

For a candidate potential $s$, define the action gap
\begin{equation*}
    \operatorname{AG}_W(s,s^\star)
    :=
    \frac12
    \int_0^1
    \int_\Omega
    q_t(x)
    \left\Vert
    \nabla s_t(x)
    -
    \nabla s_t^\star(x)
    \right\Vert_{g_x^\Omega}^2
    dx\,dt.
\end{equation*}
Expanding the square gives the cross term
$
    \int_\Omega
    q_t
    g^\Omega
    \left(
    \nabla s_t,\nabla s_t^\star
    \right)
    dx.
$
Using spatial integration by parts and \eqref{eq:continuous_am_target},
\begin{equation*}
    \int_\Omega
    q_t
    g^\Omega
    \left(
    \nabla s_t,\nabla s_t^\star
    \right)
    dx
    =
    \int_\Omega
    s_t\,\dot q_t\,dx.
\end{equation*}
Integration by parts in time then yields
\begin{align}
    \int_0^1
    \int_\Omega
    s_t\,\dot q_t\,dx\,dt
    =
    \int_\Omega
    s_1q_1\,dx
    -
    \int_\Omega
    s_0q_0\,dx
      -
    \int_0^1
    \int_\Omega
    q_t\,\partial_ts_t\,dx\,dt.
    \label{eq:continuous_temporal_ibp}
\end{align}
Substituting \eqref{eq:continuous_temporal_ibp} into the expanded action gap gives
$
    \operatorname{AG}_W(s,s^\star)
    =
    \mathcal L_{\mathrm{AM}}(s)
    -
    \mathcal L_{\mathrm{AM}}(s^\star),
$
where
\begin{align}
    \mathcal L_{\mathrm{AM}}(s)
    :=
     &
    \mathbb E_{X\sim q_0}[s_0(X)]
    -
    \mathbb E_{X\sim q_1}[s_1(X)]
    \nonumber \\
     & +
    \int_0^1
    \mathbb E_{X\sim q_t}
    \left[
        \partial_t s_t(X)
        +
        \frac12
        \Vert\nabla s_t(X)\Vert_{g_X^\Omega}^2
        \right]
    dt.
    \label{eq:continuous_am_loss_appendix}
\end{align}
This is the Action Matching identity of \citet[Theorem~2.2]{neklyudov2023action}. Its significance for learning from snapshots is that the prescribed path enters \eqref{eq:continuous_am_loss_appendix} only through expectations under its marginals. The discrete construction below follows the same fixed-path principle, but its kinetic energy contains an additional density-dependent edge mobility.

\subsection{Dynamical optimal transport on discrete sample spaces}
\label{subsec:discrete_optimal_transport}

We now replace the continuous state space by a finite set
$
    I=\{1,\ldots,N\}
$
equipped with a connected undirected weighted graph
$
    G=(I,E,\omega),
$
where $E$ is the set of unoriented edges and
\begin{equation*}
    \omega_{ij}
    =
    \omega_{ji}
    >
    0,
    \qquad
    \{i,j\}\in E.
\end{equation*}
The graph specifies which pairs of states may exchange probability locally. Such graph-based differential structures underlie several formulations of discrete dynamical optimal transport \citep{maas2011gradient,chow2012fokker,mielke2011gradient,li2018natural}. We write
$
    \mathcal N(i)
    :=
    \{j\in I:\{i,j\}\in E\}
$
for the neighborhood of $i$. The probability simplex and its relative interior are
\begin{align}
    \mathcal P(I)
     & :=
    \left\{
    p\in\mathbb R^N:
    p_i\geq0,\;
    \mathbf1^\top p=1
    \right\},
    \\
    \mathcal P_+(I)
     & :=
    \left\{
    p\in\mathcal P(I):
    p_i>0,\ i\in I
    \right\}.
\end{align}

\paragraph{Graph gradient and divergence.}
A graph vector field is a skew-symmetric edge field
\begin{equation*}
    v_{ij}
    =
    -v_{ji},
    \qquad
    \{i,j\}\in E.
\end{equation*}
For a scalar potential $S\in\mathbb R^N$, define the graph gradient
\begin{equation*}
    (\nabla_GS)_{ij}
    :=
    \sqrt{\omega_{ij}}
    (S_i-S_j).
\end{equation*}
This is the weighted finite-difference convention used in graph formulations of discrete Wasserstein geometry \citep{chow2012fokker,li2018natural,li2022transport}. The graph divergence is
\begin{equation*}
    (\operatorname{div}_Gv)_i
    :=
    \sum_{j\in\mathcal N(i)}
    \sqrt{\omega_{ij}}\,
    v_{ji}.
\end{equation*}
With these signs, $-\operatorname{div}_G$ is the adjoint of $\nabla_G$. Indeed, for any potential $S$ and skew-symmetric edge field $v$,
\begin{align}
    S^\top[-\operatorname{div}_Gv]
     & =
    -\sum_{i\in I}
    S_i
    \sum_{j\in\mathcal N(i)}
    \sqrt{\omega_{ij}}\,v_{ji}
    \nonumber \\
     & =
    \frac12
    \sum_{i\in I}
    \sum_{j\in\mathcal N(i)}
    \sqrt{\omega_{ij}}
    (S_i-S_j)v_{ij}
    \nonumber \\
     & =
    \sum_{\{i,j\}\in E}
    (\nabla_GS)_{ij}v_{ij}.
    \label{eq:discrete_integration_by_parts_general}
\end{align}
The factor $1/2$ in the directed double sum compensates for counting every unoriented edge twice. This identity is the finite-graph analogue of integration by parts and appears throughout discrete transport calculus \citep{chow2012fokker,chow2019discrete,rancati2026learning}.

\paragraph{Mobility-weighted transport.}
In the continuous continuity equation, the pointwise density multiplies a velocity defined at the same spatial location. On a graph, probability masses live on vertices while velocities live on edges. A discrete transport geometry must therefore specify how the two endpoint masses determine the amount of probability available for transport across an edge. This role is played by a mobility or mean \citep{maas2011gradient,mielke2011gradient,erbar2019geometry}.

\begin{definition}[Admissible mobility]
    \label{def:admissible_mean}
    We adopt the following sufficient common assumptions, rather than identify them with every discrete transport framework: $\theta:[0,\infty)^2\to[0,\infty)$ is continuous, smooth and strictly positive on $(0,\infty)^2$, symmetric, nondecreasing in each argument, jointly concave, normalized by $\theta(1,1)=1$, and positively homogeneous:
    \begin{equation*}
        \theta(\lambda a,\lambda b)=\lambda\theta(a,b),\qquad\lambda>0.
    \end{equation*}
    We do not require $\theta(0,b)=0$, unlike the zero-boundary class in \citet[Assumption~3.1]{maas2011gradient}; arithmetic mobility is therefore allowed. Connectedness and positive conductances suffice for the tangent representation; temporal regularity supports the action identity; homogeneity enables ratio factorization. The softplus minimizer itself does not depend on the mobility.

\end{definition}

The normalization and homogeneity imply $\theta(a,a)=a$. Standard examples include the arithmetic mean
\begin{equation*}
    \theta^{\mathrm{ari}}(a,b)
    =
    \frac{a+b}{2},
\end{equation*}
the logarithmic mean
\begin{equation*}
    \theta^{\log}(a,b)
    =
    \begin{cases}
        \dfrac{a-b}{\log a-\log b},
         & a\neq b, \\[1ex]
        a,
         & a=b,
    \end{cases}
\end{equation*}
and the geometric mean
\begin{equation*}
    \theta^{\mathrm{geo}}(a,b)
    =
    \sqrt{ab}.
\end{equation*}
We also use the harmonic mean $\theta^{\mathrm{har}}(a,b)=2ab/(a+b)$, extended by zero when either argument vanishes. Different choices induce different discrete transport geometries. In the stationary-relative reversible-chain construction, the logarithmic mean makes the heat flow a gradient flow of relative entropy \citep{maas2011gradient,erbar2012ricci}.

For $p\in\mathcal P_+(I)$, write
$
    \theta_{ij}(p)
    :=
    \theta(p_i,p_j).
$
The mobility-weighted edge pairing is
\begin{equation*}
    \langle v,\widetilde v\rangle_{p,\theta}
    :=
    \sum_{\{i,j\}\in E}
    \theta_{ij}(p)
    v_{ij}\widetilde v_{ij}.
\end{equation*}
In particular,
\begin{equation*}
    \left|
    \nabla_GS
    \right|_{p,\theta}^2
    :=
    \sum_{\{i,j\}\in E}
    \omega_{ij}
    \theta(p_i,p_j)
    (S_i-S_j)^2.
\end{equation*}
Taking
$
    v
    =
    \theta(p)\nabla_G R
$
in \eqref{eq:discrete_integration_by_parts_general} yields the mobility-weighted integration-by-parts identity
\begin{equation}
    \langle
    \nabla_GS,
    \nabla_GR
    \rangle_{p,\theta}
    =
    S^\top
    \left[
        -\operatorname{div}_G
        \left(
        \theta(p)\nabla_GR
        \right)
        \right].
    \label{eq:discrete_integration_by_parts}
\end{equation}
This is the discrete counterpart of
\begin{equation*}
    \int q\,g^\Omega(\nabla s,\nabla r)
    =
    \int
    s[-\operatorname{div}(q\nabla r)].
\end{equation*}

A differentiable probability path $t\mapsto p(t)$ generated by a potential $S(t)$ then satisfies the graph continuity equation
\begin{equation}
    \dot p(t)
    +
    \operatorname{div}_G
    \left(
    \theta(p(t))\nabla_GS(t)
    \right)
    =
    0.
    \label{eq:discrete_continuity}
\end{equation}
In coordinates,
\begin{equation*}
    \dot p_i(t)
    =
    \sum_{j\in\mathcal N(i)}
    \omega_{ij}
    \theta
    \left(
    p_i(t),p_j(t)
    \right)
    \left(
    S_i(t)-S_j(t)
    \right).
\end{equation*}
Graph continuity equations of this form are central to the dynamical transport constructions of \citet{maas2011gradient,chow2012fokker,mielke2011gradient}. Compared with the continuous equation
$\dot q=-\operatorname{div}(q\nabla s)$, the pointwise density is replaced by an edge mobility because the state and velocity variables live on vertices and edges, respectively.

\paragraph{Kinetic action and discrete dynamical transport.}
The instantaneous kinetic energy associated with $p$ and $S$ is
\begin{equation}
    \mathcal K_{G,\theta}(p,S)
    :=
    \frac12
    \left|
    \nabla_GS
    \right|_{p,\theta}^2
    =
    \frac12
    \sum_{\{i,j\}\in E}
    \omega_{ij}
    \theta(p_i,p_j)
    (S_i-S_j)^2.
    \label{eq:discrete_kinetic_energy}
\end{equation}
For a path satisfying \eqref{eq:discrete_continuity}, its kinetic action is
\begin{equation*}
    \mathcal A_{G,\theta}[p,S]
    :=
    \int_0^1
    \mathcal K_{G,\theta}
    \left(
    p(t),S(t)
    \right)
    dt.
\end{equation*}
For positive endpoints $p^0,p^1$, define the interior-path distance using absolutely continuous $p:[0,1]\to\mathcal P_+(I)$ and measurable potentials of finite kinetic action satisfying the continuity equation almost everywhere. The graph analogue of the Benamou--Brenier problem is then
\begin{equation}
    \frac12
    W_{G,\theta}^2(p^0,p^1)
    :=
    \inf_{\substack{
            p,S\\
            \dot p+
            \operatorname{div}_G(\theta(p)\nabla_GS)=0\\
            p(0)=p^0,\;
            p(1)=p^1
        }}
    \mathcal A_{G,\theta}[p,S].
    \label{eq:discrete_dynamic_transport}
\end{equation}
Discrete dynamical transport distances of this type were introduced and studied in various closely related forms by
\citet{maas2011gradient,chow2012fokker,mielke2011gradient}; the precise global behavior near the boundary of the simplex depends on the mobility. Our method uses the induced geometry on $\mathcal P_+(I)$, where all edge mobilities are strictly positive.

\begin{remark}[Relation to static Kantorovich transport]
    \label{rem:discrete_dynamic_vs_kantorovich}
    The distance in \eqref{eq:discrete_dynamic_transport} should not be confused with applying the classical Kantorovich problem directly to the finite metric space $(I,d^G)$. The latter gives
    \begin{equation*}
        W_{\mathrm K}^2(p^0,p^1)
        :=
        \min_{\pi\in\Pi(p^0,p^1)}
        \sum_{i,j\in I}
        d^G(i,j)^2\pi_{ij},
    \end{equation*}
    where $\Pi(p^0,p^1)$ is the set of couplings with the prescribed marginals \citep{kantorovich1942translocation,villani2009optimal}. In general,
    \begin{equation*}
        W_{G,\theta}
        \neq
        W_{\mathrm K}.
    \end{equation*}
    The static Kantorovich problem is determined by the ground cost and a coupling of the endpoints, whereas the dynamical graph construction additionally specifies how probability moves infinitesimally through the graph and depends on the choice of mobility $\theta$
    \citep{maas2011gradient,erbar2014gradient}. It is this latter infinitesimal geometry that is required for the fixed-path Action Matching construction.
\end{remark}

The continuity equation \eqref{eq:discrete_continuity} and kinetic energy \eqref{eq:discrete_kinetic_energy} already contain the two ingredients needed for Action Matching. We next make the resulting tangent-space geometry explicit and show that every tangent direction has a unique minimum-energy gradient representative.

\subsection{Wasserstein geometry on the probability simplex}
\label{subsec:wasserstein_geometry_on_simplex}

The relative interior $\mathcal P_+(I)$ is an $(N-1)$-dimensional smooth manifold with tangent space \citep{amari2000methods,amari2016information,ay2017information,pistone2019information}.
\begin{equation*}
    T_p\mathcal P_+(I)
    =
    \left\{
    \sigma\in\mathbb R^N:
    \mathbf1^\top\sigma=0
    \right\}.
\end{equation*}
Discrete dynamical transport equips this manifold with a density-dependent, mobility-dependent Riemannian metric \citep{maas2011gradient,chow2012fokker,mielke2011gradient,li2022transport,li2025geometric}, rather than the Fisher--Rao geometry commonly considered in information geometry. We derive this transport metric directly for the general mobility $\theta$ used throughout the paper.

\paragraph{Mobility-weighted graph Laplacian.}
Fix an arbitrary orientation of the unoriented edges and let
$
    D\in\mathbb R^{|E|\times N}
$
be the weighted incidence matrix
\begin{equation*}
    D_{e,k}
    =
    \begin{cases}
        \sqrt{\omega_{ij}},
         & k=i,              \\
        -\sqrt{\omega_{ij}},
         & k=j,              \\
        0,
         & \text{otherwise},
    \end{cases}
    \qquad
    e=(i,j).
\end{equation*}
Thus $DS$ is the oriented representation of $\nabla_GS$. For $p\in\mathcal P_+(I)$ define the diagonal mobility matrix
\begin{equation*}
    \Theta_\theta(p)
    :=
    \operatorname{diag}
    \left(
    \theta(p_i,p_j)
    \right)_{\{i,j\}\in E}
\end{equation*}
and the mobility-weighted graph Laplacian
\begin{equation}
    L_\theta(p)
    :=
    D^\top
    \Theta_\theta(p)
    D.
    \label{eq:discrete_weighted_laplacian}
\end{equation}
The operator does not depend on the arbitrary edge orientation. With the sign conventions of Section~\ref{subsec:discrete_optimal_transport},
\begin{equation*}
    L_\theta(p)S
    =
    -
    \operatorname{div}_G
    \left(
    \theta(p)\nabla_GS
    \right).
\end{equation*}
Hence the graph continuity equation can be written compactly as
$
    \dot p(t)
    =
    L_\theta(p(t))S(t).
$

For any $S,R\in\mathbb R^N$,
\begin{align}
    S^\top L_\theta(p)R
     =
    (DS)^\top
    \Theta_\theta(p)
    (DR)
    =
    \langle
    \nabla_GS,
    \nabla_GR
    \rangle_{p,\theta}.
    \label{eq:discrete_laplacian_pairing}
\end{align}
Since $p\in\mathcal P_+(I)$ and $\theta$ is positive on $(0,\infty)^2$, all diagonal entries of $\Theta_\theta(p)$ are strictly positive. Therefore
\begin{equation*}
    S^\top L_\theta(p)S
    =
    \sum_{\{i,j\}\in E}
    \omega_{ij}
    \theta(p_i,p_j)
    (S_i-S_j)^2
    \geq0.
\end{equation*}
Moreover,
$
    S^\top L_\theta(p)S=0
$
if and only if $S_i=S_j$ on every edge. Since $G$ is connected,
\begin{equation}
    \ker L_\theta(p)
    =
    \operatorname{span}\{\mathbf1\}.
    \label{eq:discrete_laplacian_kernel}
\end{equation}
Thus
$
    \operatorname{rank}L_\theta(p)
    =
    N-1
$
and, because $L_\theta(p)\mathbf1=0$,
\begin{equation}
    \operatorname{Ran}L_\theta(p)
    =
    T_p\mathcal P_+(I).
    \label{eq:discrete_laplacian_range}
\end{equation}
This is the finite-dimensional mechanism behind the identification of tangents with discrete gradients in the Maas--Erbar--Chow geometry
\citep{maas2011gradient,chow2019discrete,rancati2026learning}.

Potentials therefore naturally live in the quotient
$
    \mathbb R^N/
    \operatorname{span}\{\mathbf1\},
$
and
$
    [S]
    \longmapsto
    L_\theta(p)S
$
is an isomorphism from this quotient to
$T_p\mathcal P_+(I)$.

\paragraph{Canonical tangent representation and graph Hodge decomposition.}
We review the canonical tangent representation in \eqref{eq:canonical_discrete_tangent} and its minimum-energy characterization in \eqref{eq:canonical_discrete_minimum_energy}. The argument uses the mobility-weighted graph Hodge decomposition \citep[Lemma~3.6]{maas2011gradient}; see also \citet[Lemma~4]{chow2019discrete}.

\begin{proof}[Derivation of the canonical tangent representation]
    Let
    $
        \sigma\in T_p\mathcal P_+(I).
    $
    By \eqref{eq:discrete_laplacian_range}, there exists
    $S^\star\in\mathbb R^N$ such that
    $
        L_\theta(p)S^\star
        =
        \sigma,
    $
    or equivalently,
    \begin{equation}
        \sigma
        =
        -
        \operatorname{div}_G
        \left(
        \theta(p)
        \nabla_GS^\star
        \right).
        \label{eq:appendix_canonical_tangent}
    \end{equation}
    If $R$ is another solution, then
    $
        L_\theta(p)(R-S^\star)=0.
    $
    Equation~\eqref{eq:discrete_laplacian_kernel} implies
    $
        R-S^\star
        =
        c\mathbf1,
    $
    so the potential class $[S^\star]$ and the gradient
    $\nabla_GS^\star$ are unique. It remains to establish the minimum-energy property. Let $v$ be any skew-symmetric edge field satisfying
    \begin{equation}
        \sigma
        =
        -
        \operatorname{div}_G
        \left(
        \theta(p)v
        \right).
        \label{eq:appendix_arbitrary_velocity_constraint}
    \end{equation}
    Define
    $
        u
        :=
        v-\nabla_GS^\star.
    $
    Subtracting \eqref{eq:appendix_canonical_tangent} from
    \eqref{eq:appendix_arbitrary_velocity_constraint} gives
    $
        \operatorname{div}_G
        \left(
        \theta(p)u
        \right)
        =
        0.
    $
    By the mobility-weighted integration-by-parts identity
    \eqref{eq:discrete_integration_by_parts},
    \begin{equation*}
        \left\langle
        \nabla_GS^\star,u
        \right\rangle_{p,\theta}
        =
        S^{\star\top}
        \left[
            -\operatorname{div}_G
            \left(
            \theta(p)u
            \right)
            \right]
        =
        0.
    \end{equation*}
    Therefore
    \begin{align}
        \frac12
        \Vert v\Vert_{p,\theta}^2
         & =
        \frac12
        \Vert
        \nabla_GS^\star+u
        \Vert_{p,\theta}^2
        \nonumber \\
         & =
        \frac12
        \Vert
        \nabla_GS^\star
        \Vert_{p,\theta}^2
        +
        \frac12
        \Vert u\Vert_{p,\theta}^2
        \nonumber \\
         & \geq
        \frac12
        \Vert
        \nabla_GS^\star
        \Vert_{p,\theta}^2.
        \label{eq:appendix_discrete_hodge_energy}
    \end{align}
    Strict positivity of the mobility on $\mathcal P_+(I)$ implies equality only when $u=0$. Hence $\nabla_GS^\star$ is the unique minimum-energy edge velocity inducing $\sigma$.
\end{proof}

Equation~\eqref{eq:appendix_discrete_hodge_energy} is the discrete analogue of the weighted Hodge decomposition
\eqref{eq:continuous_hodge_energy}. In particular, an arbitrary graph velocity can be decomposed as
\begin{equation*}
    v
    =
    \nabla_GS^\star+u,
    \qquad
    \operatorname{div}_G(\theta(p)u)=0,
\end{equation*}
where the second component is invisible to the marginal tangent but contributes additional kinetic energy. This is precisely why the scalar potential in Discrete Action Matching is a geometric selection principle rather than merely a parameterization choice.

\paragraph{Primal representation of the metric.}
Let $L_\theta(p)^\dagger$ denote the Moore--Penrose pseudoinverse. Since $\ker L_\theta(p)=\operatorname{span}\{\mathbf1\}$, the pseudoinverse acts as the inverse of $L_\theta(p)$ on the tangent space. For
\begin{equation*}
    \sigma
    =
    L_\theta(p)S,
    \qquad
    \tau
    =
    L_\theta(p)R,
\end{equation*}
we have
\begin{align}
    \sigma^\top
    L_\theta(p)^\dagger
    \tau
     & =
    S^\top
    L_\theta(p)
    L_\theta(p)^\dagger
    L_\theta(p)
    R
    \nonumber \\
     & =
    S^\top
    L_\theta(p)R
    \nonumber \\
     & =
    \langle
    \nabla_GS,\nabla_GR
    \rangle_{p,\theta}.
    \label{eq:discrete_dual_primal_equivalence}
\end{align}
Thus the same metric can be written either in potential coordinates,
\begin{equation*}
    g_{p,\theta}^G
    \left(
    L_\theta(p)S,
    L_\theta(p)R
    \right)
    :=
    \langle
    \nabla_GS,\nabla_GR
    \rangle_{p,\theta},
\end{equation*}
or directly in tangent coordinates,
\begin{equation*}
    g_{p,\theta}^G(\sigma,\tau)
    :=
    \sigma^\top
    L_\theta(p)^\dagger
    \tau.
\end{equation*}
Because $\theta$ is smooth and positive on the interior and
$L_\theta(p)$ has constant rank $N-1$ on $\mathcal P_+(I)$, this defines a smooth positive-definite inner product on each tangent space. Hence $\mathcal P_+(I)$ carries the finite-dimensional Riemannian structure associated with the graph transport geometry; related constructions and their global metric properties are developed in
\citet{maas2011gradient,mielke2011gradient,li2022transport,li2025geometric}.

The kinetic energy of a tangent vector is therefore
\begin{equation*}
    \mathcal T_{G,\theta}(p,\sigma)
    =
    \frac12
    g_{p,\theta}^G(\sigma,\sigma)
    =
    \frac12
    \sigma^\top
    L_\theta(p)^\dagger
    \sigma,
\end{equation*}
where $\sigma=L_\theta(p)S$. Consequently, the dynamical transport problem \eqref{eq:discrete_dynamic_transport} admits the following primal form over the same interior path class
\begin{equation*}
    \frac12
    W_{G,\theta}^2(p^0,p^1)
    =
    \inf_{\substack{
            p_t\in\mathcal P_+(I)\\
            p_0=p^0,\;
            p_1=p^1
        }}
    \frac12
    \int_0^1
    \dot p_t^\top
    L_\theta(p_t)^\dagger
    \dot p_t\,dt.
\end{equation*}
Thus the graph kinetic action is the Riemannian energy of the probability curve under the metric tensor represented by $L_\theta(p)^\dagger$ on
$T_p\mathcal P_+(I)$.

\paragraph{Discrete Action Matching identity.}
The same structure gives the discrete counterpart of the continuous Action Matching derivation. Use the assumptions and zero-mean gauge of Theorem~\ref{thm:discrete_action_matching}; write $p_t=p(t)$ and $S_t=S(t)$, so subscripts in this proof denote time, not state. Let $S_t^\star=L_\theta(p_t)^\dagger\dot p_t$. Then
\begin{equation}
    \dot p_t
    =
    L_\theta(p_t)S_t^\star
    =
    -
    \operatorname{div}_G
    \left(
    \theta(p_t)
    \nabla_GS_t^\star
    \right).
    \label{eq:appendix_dam_target}
\end{equation}
For a candidate $S_t$, expand the discrete action gap
\begin{align}
    \operatorname{AG}_{G,\theta}(S,S^\star)
    =
    \frac12
    \int_0^1
    \Big(
    \Vert\nabla_GS_t\Vert_{p_t,\theta}^2
    +
    \Vert\nabla_GS_t^\star\Vert_{p_t,\theta}^2
    -
    2
    \langle
    \nabla_GS_t,
    \nabla_GS_t^\star
    \rangle_{p_t,\theta}
    \Big)
    dt.
    \label{eq:appendix_dam_gap_expansion}
\end{align}
Using \eqref{eq:discrete_integration_by_parts} and
\eqref{eq:appendix_dam_target},
\begin{equation*}
    \langle
    \nabla_GS_t,
    \nabla_GS_t^\star
    \rangle_{p_t,\theta}
    =
    S_t^\top\dot p_t.
\end{equation*}
Temporal integration by parts gives
\begin{equation*}
    \int_0^1
    S_t^\top\dot p_t\,dt
    =
    S_1^\top p_1
    -
    S_0^\top p_0
    -
    \int_0^1
    p_t^\top\partial_tS_t\,dt.
\end{equation*}
Substituting into \eqref{eq:appendix_dam_gap_expansion} yields
\begin{align}
    \operatorname{AG}_{G,\theta}(S,S^\star)
    =
    p_0^\top S_0
    &-
    p_1^\top S_1
    +
    \int_0^1
    \left[
        p_t^\top\partial_tS_t
        +
        \frac12
        \Vert\nabla_GS_t\Vert_{p_t,\theta}^2
        \right]
    dt
    \nonumber \\
     & +
    \frac12
    \int_0^1
    \Vert\nabla_GS_t^\star\Vert_{p_t,\theta}^2
    dt.
    \label{eq:appendix_dam_gap_loss}
\end{align}
The last term is independent of $S$. Evaluating the same expression at $S=S^\star$ shows that
$
    \operatorname{AG}_{G,\theta}(S,S^\star)
    =
    \mathcal L_{\mathrm{DAM}}(S)
    -
    \mathcal L_{\mathrm{DAM}}(S^\star),
$
which proves Theorem~\ref{thm:discrete_action_matching}. This is the exact graph analogue of the fixed-path Action Matching identity of
\citet[Theorem~2.2]{neklyudov2023action}. For nonlinear mobilities, the remaining density dependence is that of
$\Vert\nabla_GS_t\Vert_{p_t,\theta}^2$ on
$\theta(p_i(t),p_j(t))$; Appendix~\ref{app:discrete_scores} shows how positive $1$-homogeneity reduces this dependence to neighboring density ratios.

\paragraph{Relation to the Maas--Rancati parameterization.}
Our notation follows the probability-coordinate graph formulation used in the main text. The original construction of \citet{maas2011gradient}, and the recent presentation of \citet{rancati2026learning}, use a different but closely related parameterization based on a reversible Markov kernel $K$ with invariant probability measure $\pi$. A probability measure is represented there by its density $\rho$ relative to $\pi$,
\begin{equation*}
    p_i
    =
    \pi_i\rho_i,
    \qquad
    \rho_i
    =
    \frac{p_i}{\pi_i}.
\end{equation*}
The corresponding symmetric edge conductance has the form
\begin{equation*}
    c_{ij}^{\mathrm{MR}}(\rho)
    =
    \pi_iK_{ij}
    m(\rho_i,\rho_j),
\end{equation*}
where reversibility,
$\pi_iK_{ij}=\pi_jK_{ji}$,
makes the conductance symmetric
\citep{maas2011gradient,rancati2026learning}. Expressed directly in our probability coordinates,
\begin{equation}
    c_{ij}^{\mathrm{MR}}(p)
    =
    \pi_iK_{ij}
    m
    \left(
    \frac{p_i}{\pi_i},
    \frac{p_j}{\pi_j}
    \right).
    \label{eq:maas_conductance_probability}
\end{equation}
Our construction instead writes the conductance as
\begin{equation*}
    c_{ij}(p)
    =
    \omega_{ij}
    \theta(p_i,p_j).
\end{equation*}
The two parameterizations should therefore be compared at the level of the edge conductances
$
    c_{ij}(p)
    \longleftrightarrow
    c_{ij}^{\mathrm{MR}}(p),
$
rather than by identifying $\rho_i$ with $p_i$. For a homogeneous symmetric $m$, reversibility gives $c_{ij}^{\mathrm{MR}}(p)=m(K_{ij}p_i,K_{ji}p_j)$. For the logarithmic mean and nonuniform $\pi$, its ratio to $m(p_i,p_j)$ generally depends on $p$; it cannot generally be absorbed into fixed $\omega_{ij}$. Uniform $\pi$ permits a constant rescaling, and the geometric mean has the special factorization $c_{ij}^{\mathrm{MR}}=K_{ij}\sqrt{\pi_i/\pi_j}\sqrt{p_ip_j}$. These are sufficient special cases, not an exhaustive classification. Under the conductance-level interpretation, the unique-gradient statement of
\citet[Proposition~3.26]{maas2011gradient}, restated as
\citet[Proposition~B.6]{rancati2026learning}, corresponds to the tangent representation proved above.

Table~\ref{tab:operator_dictionary} summarizes the continuous and discrete constructions used throughout the paper.

\begin{table*}[t]
    \centering
    \caption{Continuous--discrete dictionary for the dynamical transport structures underlying Action Matching.}
    \label{tab:operator_dictionary}
    \small
    \renewcommand{\arraystretch}{1.3}
    \begin{tabularx}{\textwidth}{
        >{\raggedright\arraybackslash}p{0.17\textwidth}
        >{\raggedright\arraybackslash}X
        >{\raggedright\arraybackslash}X
        }
        \toprule
        \textbf{Role}
         &
        \textbf{Continuous sample space}
         &
        \textbf{Finite graph}
        \\
        \midrule

        Probability state
         &
        $q\in\mathcal P_+(\Omega)$
         &
        $p\in\mathcal P_+(I)$
        \\

        Tangent vector
         &
        $\displaystyle
            \sigma\in T_q\mathcal P_+(\Omega),
            \quad
            \int_\Omega\sigma\,dx=0$
         &
        $\displaystyle
            \sigma\in T_p\mathcal P_+(I),
            \quad
            \mathbf1^\top\sigma=0$
        \\

        Potential
         &
        $\displaystyle
            [s]\in C^\infty(\Omega)/\mathbb R$
         &
        $\displaystyle
            [S]\in
            \mathbb R^N/
            \operatorname{span}\{\mathbf1\}$
        \\

        Gradient
         &
        $\displaystyle
            \nabla s$
         &
        $\displaystyle
            (\nabla_GS)_{ij}
            =
            \sqrt{\omega_{ij}}(S_i-S_j)$
        \\

        Potential-to-tangent map
         &
        $\displaystyle
            -\operatorname{div}(q\nabla s)$
         &
        $\displaystyle
            -\operatorname{div}_G
            \left(
            \theta(p)\nabla_GS
            \right)
            =
            L_\theta(p)S$
        \\

        Continuity equation
         &
        $\displaystyle
            \dot q_t
            =
            -\operatorname{div}(q_t\nabla s_t)$
         &
        $\displaystyle
            \dot p_t
            =
            L_\theta(p_t)S_t$
        \\

        Kinetic pairing
         &
        $\displaystyle
            \int_\Omega
            q\,
            g_x^\Omega(\nabla s,\nabla r)\,dx$
         &
        $\displaystyle
            \sum_{\{i,j\}\in E}
            \omega_{ij}\theta(p_i,p_j)
            (S_i-S_j)(R_i-R_j)$
        \\

        Tangent metric
         &
        $\displaystyle
            g_q^W(\sigma,\tau)$
         &
        $\displaystyle
            g_{p,\theta}^G(\sigma,\tau)
            =
            \sigma^\top
            L_\theta(p)^\dagger
            \tau$
        \\

        Kinetic energy
         &
        $\displaystyle
            \frac12
            g_q^W(\sigma,\sigma)$
         &
        $\displaystyle
            \frac12
            g_{p,\theta}^G(\sigma,\sigma)$
        \\

        Path action
         &
        $\displaystyle
            \int_0^1
            \mathcal K_W(q_t,\dot q_t)\,dt$
         &
        $\displaystyle
            \int_0^1
            \mathcal T_{G,\theta}(p_t,\dot p_t)\,dt$
        \\

        Invisible component
         &
        $\displaystyle
            \operatorname{div}(qw)=0$
         &
        $\displaystyle
            \operatorname{div}_G(\theta(p)u)=0$
        \\

        Gauge invariance
         &
        $\displaystyle
            s\mapsto s+c$
         &
        $\displaystyle
            S\mapsto S+c\mathbf1$
        \\

        \bottomrule
    \end{tabularx}
\end{table*}

The correspondence makes the common structure explicit. In both settings, a marginal tangent alone admits many velocity representations. The transport geometry projects this ambiguity onto a unique gradient representative, and its weighted squared norm defines the instantaneous kinetic energy. Continuous Action Matching exploits this structure because multiplication by $q_t$ turns the kinetic term directly into an expectation under the observed marginal. On a graph, the corresponding weight is the edge mobility $\theta(p_i,p_j)$, which depends on two unknown probability masses. The next appendix section develops the local density-ratio representation that removes this remaining obstacle.

\section{Discrete scores and local density ratios}
\label{app:discrete_scores}

The Discrete Action Matching objective depends on the current marginal through the edge mobility $\theta(p_i(t),p_j(t))$. Positive $1$-homogeneity reduces this dependence to ratios of neighboring probability masses, motivating the local density-ratio estimator introduced in Section~\ref{subsec:dam_ratio_estimation}. Such ratios have a longer history in normalization-free estimation on discrete spaces: ratio matching was introduced by \citet{hyvarinen2007some}, related constructions were placed within proper homogeneous local scoring rules by \citet{dawid2012proper}, and graph-based local Bregman formulations were developed by \citet{kanamori2017graph}. More recently, neighboring ratios were interpreted as discrete score-like quantities by \citet{meng2022concrete} and used directly in reverse-time discrete diffusion by \citet{lou2024discrete}.

In DAM, these ratios are auxiliary marginal information rather than the learned population velocity: they determine the mobility-weighted kinetic energy, whereas the action potential $S$ represents the canonical dynamics. This section provides the details omitted from the main text. We first review local density ratios and their relation to discrete scores, then derive their appearance in the DAM kinetic energy and prove consistency of the local softplus estimator. Finally, we discuss its relation to implicit score entropy and give the corresponding empirical training and Markov realization.

\subsection{Local density ratios and discrete scores}
\label{app:local_density_ratios}

Let $G=(I,E,\omega)$ be a connected undirected graph and let
$p\in\mathcal P_+(I)$. For every directed version of an edge
$\{i,j\}\in E$, define the \emph{local density ratio}
\begin{equation}
    r_p(i,j)
    :=
    \frac{p_j}{p_i}.
    \label{eq:local_density_ratio}
\end{equation}
Although the underlying graph is undirected, the ratio is directed and satisfies the reciprocity identity
$
    r_p(j,i)
    =
    r_p(i,j)^{-1}.
$
The collection
$
    r_p(i)
    :=
    \left(
    r_p(i,j)
    \right)_{j\in\mathcal N(i)}
$
therefore describes the masses of neighboring states relative to the mass at $i$. A basic property of \eqref{eq:local_density_ratio} is invariance to global normalization. If
$\widetilde p_i=cp_i$ for $c>0$, then
\begin{equation*}
    \frac{\widetilde p_j}{\widetilde p_i}
    =
    \frac{p_j}{p_i}.
\end{equation*}
This observation is central to ratio-matching and homogeneous local-scoring methods for unnormalized discrete models
\citep{hyvarinen2007some,dawid2012proper,kanamori2017graph}.

Conversely, on a connected graph the complete collection of neighboring ratios determines the normalized probability distribution. For any path
$
    i_0,i_1,\ldots,i_m
$
in $G$,
\begin{equation}
    \frac{p_{i_m}}{p_{i_0}}
    =
    \prod_{k=0}^{m-1}
    r_p(i_k,i_{k+1}).
    \label{eq:local_ratio_path_product}
\end{equation}
Connectedness therefore determines all masses relative to an arbitrary reference state, after which the constraint
$\sum_i p_i=1$
fixes their common scale. Ratios arising from a probability distribution also satisfy cycle consistency: for every cycle
$C=(i_0,\ldots,i_m=i_0)$,
\begin{equation*}
    \prod_{k=0}^{m-1}
    r_p(i_k,i_{k+1})
    =
    1.
\end{equation*}
Thus an arbitrary collection of positive directed edge values need not correspond to a globally defined probability distribution.

\paragraph{Ratio matching and local scoring rules.}
The use of neighboring probability ratios for statistical estimation predates their interpretation as discrete scores. \citet{hyvarinen2007some} introduced ratio matching as a discrete counterpart of score matching \citep{hyvarinen2005estimation}, initially for binary variables. The normalization constant cancels because the relevant conditional quantities can be expressed through ratios between probabilities of neighboring configurations.

A broader statistical interpretation was developed by
\citet{dawid2012proper}. A scoring rule assigns a loss
$\mathsf S(i,q)$ to a quoted distribution $q$ after observing
$i\in I$. It is \emph{proper} if its expected value under a distribution
$p$ is minimized by quoting $q=p$, and \emph{homogeneous} if
\begin{equation*}
    \mathsf S(i,cq)
    =
    \mathsf S(i,q),
    \qquad
    c>0.
\end{equation*}
Homogeneous scores can therefore be evaluated from an unnormalized positive mass function. \citet{dawid2012proper} characterize broad classes of proper local homogeneous scoring rules through clique-based homogeneous functions on an undirected neighborhood graph, with pseudo-likelihood and ratio matching among their examples.

\citet{kanamori2017graph} develop the corresponding viewpoint through composite local Bregman divergences. For a graph-defined neighborhood system, their local potentials depend on vectors of ratios between neighboring masses. For example, their graph extension of the ratio-matching score contains terms of the form
\begin{equation*}
    \mathsf S_{\mathrm{RM}}(i,q)
    =
    \sum_{j\in\mathcal N(i)}
    \frac{1}
    {\left(1+q_i/q_j\right)^2},
\end{equation*}
which is invariant to a global normalization of $q$. They further show that this graph ratio-matching construction is strictly proper when the neighborhood graph is connected
\citep{kanamori2017graph}. This is consistent with the elementary reconstruction property in
\eqref{eq:local_ratio_path_product}: a connected collection of exact local ratios determines the complete normalized distribution.

\paragraph{Concrete scores.}
\citet{meng2022concrete} reinterpret essentially the same local information as a discrete analogue of the continuous Stein score. For a neighborhood $\mathcal N(i)$, their \emph{Concrete score} has components
\begin{equation*}
    c_p(i,j)
    :=
    \frac{p_j-p_i}{p_i}
    =
    \frac{p_j}{p_i}-1,
    \qquad
    j\in\mathcal N(i).
\end{equation*}
Hence
\begin{equation*}
    c_p(i,j)
    =
    r_p(i,j)-1,
    \qquad
    r_p(i,j)
    =
    1+c_p(i,j),
\end{equation*}
so the Concrete score and the local density-ratio field contain exactly the same information.

This identity also explains the completeness result of
\citet[Theorem~1]{meng2022concrete}: if the neighborhood-induced graph is connected, equality of Concrete-score fields implies equality of the underlying normalized distributions. The same work relates the construction to the continuous score. For
\begin{equation*}
    \mathcal N_\delta(x)
    =
    \{x+\delta e_k\}_{k=1}^d,
\end{equation*}
\citet[Proposition~1]{meng2022concrete} show that
\begin{equation*}
    \frac{1}{\delta}
    c_p(x;\mathcal N_\delta)
    \longrightarrow
    \nabla_x\log p(x)
    \qquad
    \text{as }\delta\rightarrow0.
\end{equation*}
Concrete Score Matching then provides a sample-based objective for estimating this local score field \citep{meng2022concrete}.

\paragraph{Score entropy and SEDD.}
The unshifted ratio itself appears directly in the time reversal of a continuous-time Markov chain. If $Q_t(i,j)$ denotes the forward transition rate from $i$ to $j$ and $p(t)$ its marginal distribution, the reverse-time off-diagonal rates contain the factor
\begin{equation*}
    \overline Q_t(i,j)
    =
    \frac{p_j(t)}{p_i(t)}
    Q_t(j,i),
    \qquad
    i\neq j,
\end{equation*}
under the usual row-generator convention
\citep{kelly2011reversibility,anderson2012continuous,sun2023scorebased,lou2024discrete}. Thus estimating the reverse dynamics of a prescribed forward diffusion requires the time-dependent ratios
\begin{equation*}
    r_t(i,j)
    =
    \frac{p_j(t)}{p_i(t)}.
\end{equation*}

Motivated by this observation, \citet{lou2024discrete} parameterize these positive ratios directly and introduce \emph{score entropy}. For a positive prediction $s(i)_j$, target ratio
$r=p_j/p_i$, and positive weight $w_{ij}$, the corresponding scalar contribution is
\begin{equation*}
    w_{ij}
    \left[
        s(i)_j
        -
        r\log s(i)_j
        +
        r(\log r-1)
        \right].
\end{equation*}
The excess loss is
$
    r\log\frac{r}{s}
    -
    r+s,
$
which is the Bregman divergence generated by the negative-entropy potential
$
    F(a)
    =
    a\log a-a.
$
Consequently, its unique optimum is $s=r$. In particular,
\citet[Proposition~3.2]{lou2024discrete} establish consistency of score entropy for the density-ratio target under full support and positive weights.

The population score-entropy objective itself contains the unknown ratios. \citet[Proposition~3.3]{lou2024discrete} eliminate them to obtain \emph{implicit score entropy},
\begin{equation}
    \mathcal L_{\mathrm{ISE}}(s;p)
    =
    \mathbb E_{i\sim p}
    \left[
        \sum_{j\neq i}
        \left(
        w_{ij}s(i)_j
        -
        w_{ji}\log s(j)_i
        \right)
        \right],
    \label{eq:implicit_score_entropy_background}
\end{equation}
up to a constant independent of the model. For the structured noising processes used in SEDD, \citet[Theorem~3.4]{lou2024discrete} instead derive a denoising form based on the known forward transition kernel, avoiding the costly evaluation of the implicit objective over all alternative states.

These constructions establish that neighboring ratios are natural normalization-free statistics of a discrete marginal distribution. In SEDD, however, the ratios are learned because they directly determine the reverse-time generator. In Discrete Action Matching they play a different role: they determine the density-dependent \emph{geometry} in which the canonical action is learned. The score-entropy/discrete-geometry connection is also discussed by \citet[Appendix~G]{rancati2026learning}; the derivation below explains its use for the DAM kinetic term without claiming priority for the general connection.

\subsection{Density ratios in Discrete Action Matching}
\label{app:ratios_in_dam}

We now derive the appearance of local density ratios in Discrete Action Matching directly from the graph kinetic energy. Here no ratio-valued score is imposed as a dynamical parameterization, unlike in score-based discrete diffusion. The ratios arise solely because the mobility $\theta$ is positively $1$-homogeneous.

\begin{corollary}[Sample-local kinetic energy]
    \label{cor:dam_sample_local_kinetic}
    Under the assumptions of Theorem~\ref{thm:discrete_action_matching}, assume additionally that $\theta$ is positively $1$-homogeneous. Then the kinetic energy admits the sample-local representation \eqref{eq:dam_sample_local_kinetic}.
\end{corollary}

Recall from \eqref{eq:discrete_kinetic_energy} that
\begin{equation}
    \mathcal K_{G,\theta}(p,S)
    =
    \frac12
    \sum_{\{i,j\}\in E}
    \omega_{ij}
    \theta(p_i,p_j)
    (S_i-S_j)^2.
    \label{eq:dam_kinetic_ratio_appendix}
\end{equation}
Define the \emph{log-ratio mobility profile}
$
    \kappa_\theta(u)
    :=
    \theta(1,e^u),
    \,
    u\in\mathbb R.
$
For $p_i,p_j>0$, positive $1$-homogeneity gives
\begin{align}
    \theta(p_i,p_j)
    =
    p_i
    \theta
    \left(
    1,\frac{p_j}{p_i}
    \right)
    =
    p_i
    \kappa_\theta
    \left(
    \log\frac{p_j}{p_i}
    \right).
    \label{eq:mobility_log_ratio_factorization}
\end{align}
Thus the dependence of the edge mobility on two unknown probability masses separates into the source probability $p_i$ and a function of one neighboring log-ratio.

Using \eqref{eq:mobility_log_ratio_factorization}, the undirected edge sum in
\eqref{eq:dam_kinetic_ratio_appendix} can be rewritten as a directed neighborhood sum:
\begin{align}
    \mathcal K_{G,\theta}(p,S)
     & =
    \frac14
    \sum_{i\in I}
    p_i
    \sum_{j\in\mathcal N(i)}
    \omega_{ij}
    \kappa_\theta
    \left(
    \log\frac{p_j}{p_i}
    \right)
    (S_i-S_j)^2
    \nonumber \\
     & =
    \frac14
    \mathbb E_{i\sim p}
    \left[
        \sum_{j\in\mathcal N(i)}
        \omega_{ij}
        \kappa_\theta
        \left(
        \log\frac{p_j}{p_i}
        \right)
        (S_i-S_j)^2
        \right].
    \label{eq:kinetic_log_ratio_expectation}
\end{align}
To verify the factor $1/4$, note that the directed sum contains both orientations of every unoriented edge. For $\{i,j\}\in E$, the two contributions are
\begin{align}
    \frac14
    \omega_{ij}
    p_i
    \kappa_\theta
    \left(
    \log\frac{p_j}{p_i}
    \right)
    (S_i-S_j)^2
    +
    \frac14
    \omega_{ij}
    p_j
    \kappa_\theta
    \left(
    \log\frac{p_i}{p_j}
    \right)
    (S_i-S_j)^2.
\end{align}
Each source-weighted mobility equals
$\theta(p_i,p_j)$ by homogeneity and symmetry, so their sum is precisely
$
    \frac12
    \omega_{ij}
    \theta(p_i,p_j)
    (S_i-S_j)^2.
$
This proves Corollary~\ref{cor:dam_sample_local_kinetic}.

Equation~\eqref{eq:kinetic_log_ratio_expectation} is the key statistical closure of Discrete Action Matching. Sampling $X\sim p$ accounts automatically for the source probability $p_X$. The remaining marginal information is only
$
    \left\{
    \log\frac{p_j}{p_X}
    \right\}_{j\in\mathcal N(X)},
$
rather than the complete probability vector $p$.

\paragraph{Arithmetic mobility.}
For $\theta_{\mathrm{ari}}(a,b)=(a+b)/2$, symmetry of the edge sum gives
\begin{align}
    \mathcal K_{G,\mathrm{ari}}(p,S)
    =
    \frac14\sum_{\{i,j\}\in E}
    \omega_{ij}(p_i+p_j)(S_i-S_j)^2
    =
    \frac14\mathbb E_{i\sim p}
    \left[
        \sum_{j\in\mathcal N(i)}
        \omega_{ij}(S_i-S_j)^2
    \right].
    \label{eq:arithmetic_ratio_free}
\end{align}
Each unordered edge contributes once for each sampled endpoint in the expectation, recovering its weight $p_i+p_j$. Thus the arithmetic action admits an estimator that requires marginal samples but no density ratios. This is an alternative to evaluating the general ratio-based expression.

The Markov realization in \eqref{eq:dam_generator_offdiag} still uses density ratios, since $\kappa_{\mathrm{ari}}(u)=(1+e^u)/2$. At exact ratios its off-diagonal rates are
\[
    Q_t(i,j)
    =
    \frac{\omega_{ij}}2
    \left(1+\frac{p_j(t)}{p_i(t)}\right)
    \psi_\varepsilon\!\left(S_j(t)-S_i(t)\right).
\]
The simplification of the action therefore does not remove the density dependence of this rate construction.

\paragraph{A scalar log-density representation.}
The required log-ratios have a natural globally consistent representation. Define
$
    U^\star(i)
    :=
    \log p_i+c,
$
where $c\in\mathbb R$ is arbitrary. Then
$
    U^\star(j)-U^\star(i)
    =
    \log\frac{p_j}{p_i}.
$
For a time-dependent marginal path,
\begin{equation}
    U^\star(t,i)
    =
    \log p_i(t)+c(t).
    \label{eq:time_log_density_potential}
\end{equation}
Only differences of $U^\star$ enter the method, so the time-dependent gauge $c(t)$ is immaterial.

This scalar representation automatically enforces the structural constraints of a valid ratio field. Defining
\begin{equation*}
    r_U(i,j)
    :=
    \exp
    \left(
    U(j)-U(i)
    \right),
\end{equation*}
we have identically
\begin{equation*}
    r_U(j,i)
    =
    r_U(i,j)^{-1}
\end{equation*}
and
\begin{equation*}
    \prod_{(i,j)\in C}
    r_U(i,j)
    =
    1
\end{equation*}
for every graph cycle $C$. Thus, unlike a free directed-edge parameterization, a scalar field cannot produce mutually inconsistent ratios around a cycle.

\paragraph{Local softplus ratio estimation.}
We estimate the field \eqref{eq:time_log_density_potential} directly from samples. For a fixed full-support marginal $p$, define
\begin{equation}
    \mathcal R_p(U)
    :=
    \mathbb E_{i\sim p}
    \left[
        \sum_{j\in\mathcal N(i)}
        \omega_{ij}
        \operatorname{softplus}
        \left(
        U_j-U_i
        \right)
        \right],
    \label{eq:softplus_ratio_population}
\end{equation}
where
$
    \operatorname{softplus}(z)
    :=
    \log(1+e^z).
$
Like the local scoring objectives reviewed above, \eqref{eq:softplus_ratio_population} depends on $p$ only through the outer expectation and therefore requires neither the individual masses $p_i$ nor their normalization constant during empirical optimization.

The population target can be characterized edge by edge.

\begin{proposition}[Consistency of the local softplus ratio loss]
    \label{prop:softplus_ratio_consistency}
    Let $G$ be connected and let $p\in\mathcal P_+(I)$. Then
    \begin{equation}
        \operatorname*{argmin}_{U\in\mathbb R^N}
        \mathcal R_p(U)
        =
        \left\{
        \log p+c\mathbf1:
        c\in\mathbb R
        \right\}.
        \label{eq:softplus_ratio_argmin}
    \end{equation}
    In particular, every minimizer satisfies
    \begin{equation*}
        U_j-U_i
        =
        \log\frac{p_j}{p_i},
        \qquad
        \{i,j\}\in E.
    \end{equation*}
\end{proposition}

\begin{proof}
    Consider an unoriented edge $\{i,j\}\in E$ and write
    $
        \delta_{ij}
        :=
        U_j-U_i.
    $
    The two directed occurrences of this edge in
    \eqref{eq:softplus_ratio_population} contribute
    \begin{equation*}
        \ell_{ij}(\delta_{ij})
        =
        \omega_{ij}
        \left[
            p_i
            \operatorname{softplus}(\delta_{ij})
            +
            p_j
            \operatorname{softplus}(-\delta_{ij})
            \right].
    \end{equation*}
    Differentiation gives
    \begin{equation*}
        \ell_{ij}'(\delta)
        =
        \omega_{ij}
        \frac{
            p_i e^\delta-p_j
        }{
            1+e^\delta
        }.
    \end{equation*}
    Therefore
    \begin{equation*}
        \ell_{ij}'(\delta)=0
        \quad\Longleftrightarrow\quad
        \delta
        =
        \log\frac{p_j}{p_i}.
    \end{equation*}
    Moreover,
    \begin{equation*}
        \ell_{ij}''(\delta)
        =
        \omega_{ij}
        \frac{
            (p_i+p_j)e^\delta
        }{
            (1+e^\delta)^2
        }
        >
        0,
    \end{equation*}
    so this stationary point is the unique minimum of the edge contribution.

    The collection
    $
        U_i^\star
        =
        \log p_i
    $
    simultaneously realizes the unique minimizing difference on every edge. Hence any global minimizer must have
    $
        U_j-U_i
        =
        U_j^\star-U_i^\star
    $
    for every edge. Since $G$ is connected, $U-U^\star$ is constant on $I$, proving
    \eqref{eq:softplus_ratio_argmin}.
\end{proof}

Proposition~\ref{prop:softplus_ratio_consistency} explains why we parameterize the ratio estimator through a scalar field $U_\phi(t,i)$ rather than independent directed predictions. At the population optimum,
\begin{equation*}
    \exp
    \left(
    U_\phi(t,j)-U_\phi(t,i)
    \right)
    =
    \frac{p_j(t)}{p_i(t)},
\end{equation*}
while positivity, reciprocity, and cycle consistency hold automatically for arbitrary model parameters.

\paragraph{Relation to implicit score entropy.}
The local softplus loss is related to, but distinct from, implicit score entropy. To make the comparison explicit, restrict
\eqref{eq:implicit_score_entropy_background} to graph neighbors, take symmetric weights
$
    w_{ij}
    =
    w_{ji}
    =
    \omega_{ij},
$
and use the coherent positive-score parameterization
$
    s_U(i)_j
    :=
    \exp(U_j-U_i).
$
Then
$
    -\log s_U(j)_i
    =
    U_j-U_i,
$
and the implicit score-entropy objective becomes
\begin{equation*}
    \mathcal L_{\mathrm{ISE}}(U;p)
    =
    \mathbb E_{i\sim p}
    \left[
        \sum_{j\in\mathcal N(i)}
        \omega_{ij}
        \left(
        e^{U_j-U_i}
        +
        U_j-U_i
        \right)
        \right].
\end{equation*}
For one edge and
$\delta=U_j-U_i$, its two orientations contribute
\begin{equation*}
    \omega_{ij}
    \left[
        p_i
        \left(
        e^\delta+\delta
        \right)
        +
        p_j
        \left(
        e^{-\delta}-\delta
        \right)
        \right].
\end{equation*}
The derivative vanishes uniquely at
$
    \delta
    =
    \log\frac{p_j}{p_i},
$
and the edge contribution is strictly convex. Consequently,
\begin{equation*}
    \operatorname*{argmin}_U
    \mathcal L_{\mathrm{ISE}}(U;p)
    =
    \operatorname*{argmin}_U
    \mathcal R_p(U)
    =
    \left\{
    \log p+c\mathbf1
    \right\}.
\end{equation*}

Thus the two objectives are \emph{target-equivalent} under the scalar log-density parameterization, but they are not the same estimator. Their nonlinear terms are respectively $e^{U_j-U_i}$ and $\operatorname{softplus}(U_j-U_i)$, and hence their gradients and finite-sample behavior differ. In particular, implicit score entropy inherits an exponential dependence on the predicted log-ratio, whereas the derivative of the softplus term is logistic and bounded. Our method uses
\eqref{eq:softplus_ratio_population}; the relation to
\citet{lou2024discrete} is therefore a connection through the common population density-ratio target, rather than an identity of training losses.

With the local log-density field estimated, the kinetic energy becomes sample-computable through
\begin{equation*}
    \mathcal K_{G,\theta}(p,S)
    =
    \frac14
    \mathbb E_{i\sim p}
    \left[
        \sum_{j\in\mathcal N(i)}
        \omega_{ij}
        \kappa_\theta
        \left(
        U^\star_j-U^\star_i
        \right)
        (S_i-S_j)^2
        \right].
\end{equation*}
We next use this representation to construct the empirical two-stage objective and the corresponding Markov dynamics.

    \subsection{Empirical training and Markov realization}
    \label{app:dam_empirical_and_sampling}

    \paragraph{Why declare an interpolant?}
    The population action integrates expectations over all $t\in[0,1]$. Finite snapshots do not supply samples at arbitrary intermediate times. Equation~\eqref{eq:declared_interpolant} declares the distribution sampled at such times. It interpolates distributions, not matched individual trajectories. Conditional on the data, sample $T$ uniformly, choose the reference with probability $\delta$, otherwise choose the left or right snapshot with probabilities $1-h(u)$ and $h(u)$, and sample a state. At endpoints use the corresponding empirical/reference mixture. The resulting state law is exactly $\widetilde p(T)$, with $\widetilde p_i(t)\geq\delta/N$.

    The original AM algorithm assumes intermediate-time sample access \citep[Algorithm~1]{neklyudov2023action}. Merely replacing its time integral by derivative evaluations at finitely many observed knots changes the optimization problem. In particular, a candidate time-only gauge $c(t)\mathbf1$ changes that fixed-knot loss by
    $
        c(0)-c(1)+\sum_k\alpha_k c'(t_k).
    $
    For $A>0$, choose
    $
        c_A(t)=A\int_0^t\prod_k(s-t_k)^2\,ds.
    $
    Every sampled derivative vanishes while $c_A(1)>0$, so the unrestricted fixed-knot objective is unbounded below. All spatial differences, actions and rates are unchanged. Spatial zero-mean gauge fixing removes this example but not generally time-dependent spatial perturbations that vanish at the knots. This counterexample does not assert divergence of every restricted neural optimizer, nor that interpolation is the only possible well-posed temporal formulation. DICE studies related finite-time AM difficulties \citep[Section~6]{blickhan2025dice}. We use the continuous-time estimator for the declared path, whose expected boundary and derivative terms cancel under a smooth time gauge; an individual Monte Carlo batch need not be gauge invariant.

    \paragraph{Regularity of the practical path.}
    A globally $C^2$ positive path satisfies Theorem~\ref{thm:discrete_action_matching}. The cubic smoothstep construction is $C^1$ and piecewise smooth, generally not globally $C^2$. On each interval the theorem's calculation applies. For continuous, piecewise $C^1$ candidates, integration-by-parts boundary terms telescope because the path and candidate are continuous. The zero-mean canonical potential is continuous and piecewise $C^1$ for smoothstep, since the path tangent vanishes on both sides of each knot. Thus the identity extends in this piecewise class. A linear probability interpolant is continuous but can have tangent jumps; its canonical potential may jump. The identity still holds with the target kinetic energy as a finite additive constant, but a continuous network need not represent that discontinuous target exactly. Endpoint matching and regularity must be checked for the actual schedule, not inferred from a continuous-time network input.

    \paragraph{Stage 1 and its guarantee.}
    At a fixed positive marginal the local softplus loss in \eqref{eq:softplus_ratio_population} has the minimizers in Proposition~\ref{prop:softplus_ratio_consistency}. For a snapshot-only population loss $\sum_k w_k\mathcal R_{p(t_k)}(U(t_k))$, assume $w_k>0$, simultaneous realizability at all knots, and global population minimization; the conclusion then holds only at those knots:
    \begin{equation*}
        U_{\phi^\star}(t_k,j)-U_{\phi^\star}(t_k,i)=\log\frac{p_j(t_k)}{p_i(t_k)}.
    \end{equation*}
    The implemented objective instead estimates
    $
        \mathcal R_{\widetilde p}(\phi)=\int_0^1\mathcal R_{\widetilde p(t)}(U_\phi(t))\,dt.
    $
    Its unrestricted optimum recovers $\log\widetilde p(t)$ modulo a gauge almost everywhere, with the same realizability and exact-minimization qualifications for a model class. This is not a finite-data recovery theorem for $p(t)$. Zero-count empirical states make unrestricted unsmoothed edge-ratio estimation problematic; reference mixing supplies positive support, at the cost of changing the target. Atlas pseudocounts, empirical resampling, and training reference mixing are separate operations (Appendix~\ref{app:gastrulation_experiment}).

    \paragraph{Frozen geometry and conditional unbiasedness.}
    For a fixed $U$ and any positive path $p$, define
    \begin{align}
        \widetilde c_{ij}^{\,U}(p)    & =\frac{\omega_{ij}}2\left[p_i\kappa_\theta(U_j-U_i)+p_j\kappa_\theta(U_i-U_j)\right],\label{eq:perturbed_conductance} \\
        \widetilde{\mathcal K}_U(p,S) & =\frac12\sum_{\{i,j\}\in E}\widetilde c_{ij}^{\,U}(p)(S_i-S_j)^2.\label{eq:perturbed_kinetic_energy}
    \end{align}
    These positive symmetric conductances describe the training energy. They alone do not ensure the same sampling current. Indeed, set $a=\omega_{ij}p_i\kappa_\theta(U_j-U_i)$, $b=\omega_{ij}p_j\kappa_\theta(U_i-U_j)$ and $z=S_j-S_i$. The source-first training current is $(a+b)z/2$, whereas the rate construction yields
    \begin{equation*}
        a\psi_\varepsilon(z)-b\psi_\varepsilon(-z)=\frac{a+b}{2}z+\frac{a-b}{2}\big[\psi_\varepsilon(z)+\psi_\varepsilon(-z)\big].
    \end{equation*}
    At $z=0$ this equals $(a-b)\varepsilon\log2$, which need not vanish. At exact ratios for $p$, $a=b=\omega_{ij}\theta(p_i,p_j)$ and
    $
        \widetilde c_{ij}^{\,U^\star}(p)=\omega_{ij}\theta(p_i,p_j).
    $
    For fixed $S$ and a frozen fitted $U$, fresh draws in \eqref{eq:dam_empirical_action_loss} satisfy
    \begin{align}
        \mathbb E[\widehat{\mathcal L}_{\mathrm{DAM}}\mid\mathcal D,U]
         & =\widetilde p(0)^\top S(0)-\widetilde p(1)^\top S(1)\nonumber                                                   \\
         & \quad+\int_0^1\left[\widetilde p(t)^\top\partial_tS(t)+\widetilde{\mathcal K}_U(\widetilde p(t),S(t))\right]dt.
        \label{eq:empirical_dam_U_objective}
    \end{align}
    To verify this, average endpoint samples, then condition on $T=t$ and average the sampled state. Pairing the two orientations of each edge gives \eqref{eq:perturbed_kinetic_energy}; finally average uniform time. At exact ratios for this same $\widetilde p$ the expression is its chosen-mobility action loss. This conditional estimator identity says nothing about unbiasedness of an optimized network or identification of the unknown intermediate population path.

    A neighborhood sum $H_f(i)=\sum_j\omega_{ij}f(i,j)$ may be estimated with $J\sim q(\cdot\mid i)$, where $q(j\mid i)>0$ on every neighbor:
    \begin{equation*}
        \widehat H_f(i,J)=\frac{\omega_{iJ}}{q(J\mid i)}f(i,J),\qquad \mathbb E[\widehat H_f(i,J)\mid i]=H_f(i).
    \end{equation*}
    Independent repetitions reduce variance. The proposal is known and is not a density-ratio estimate. Freezing $U$ keeps the effective conductances fixed during action optimization, so $S$ is fitted under the geometry determined by the estimated density ratios.

    \paragraph{Markov realization and its law.}
    For finite continuous fields abbreviate $Q_t=Q_t^{S,U}$, with
    \begin{equation*}
        Q_t(i,j)=\omega_{ij}\kappa_\theta(U_j-U_i)\psi_\varepsilon(S_j-S_i)\quad(j\in\mathcal N(i)),\qquad Q_t(i,i)=-\sum_{j\ne i}Q_t(i,j).
    \end{equation*}
    Nonneighbor rates vanish. Nonnegative finite off-diagonals and zero row sums give a valid generator; continuity makes it bounded on the compact time interval. Both $S$ and $U$ may be shifted by time-dependent spatial constants without changing rates. Let $Y_t$ denote the constructed chain and $\mu(t)$ its column-vector law. Then
    $
        \dot\mu(t)=Q_t^\top\mu(t).
    $
    If $U_j-U_i=\log(p_j/p_i)$, homogeneity and the split identity give
    \begin{align}
        p_iQ_t(i,j)-p_jQ_t(j,i) & =\omega_{ij}\theta(p_i,p_j)(S_j-S_i),\label{eq:appendix_current_identity}                            \\
        (Q_t^\top p)_i                & =\sum_j\omega_{ij}\theta(p_i,p_j)(S_i-S_j)=[L_\theta(p)S]_i.\label{eq:appendix_generator_continuity}
    \end{align}
    If additionally $S=S^\star$, the prescribed path satisfies the same forward equation as $\mu$. Matching initial laws and uniqueness yield $\mu=p$. This applies intervalwise for the piecewise-smooth construction. It does not equate a target path with a simulated law under approximate fits. The algebraic current identity also holds relative to $\widehat p=\operatorname{softmax}(U)$, but that vector is not automatically $\mu$.

    The smooth split obeys $\psi_\varepsilon(z)=\max(z,0)+\varepsilon\log(1+e^{-|z|/\varepsilon})$. At exact ratios the added symmetric traffic satisfies detailed balance with $p$ and leaves its instantaneous current unchanged, in the spirit of marginal-preserving generator additions \citep{campbell2024generative}. This does not transfer an exact-current property to imperfect ratios.

    \paragraph{Inference and trajectory simulation.}
    For marginal evaluation, choose an initial distribution and time grid; freeze $Q$ at each substep's evaluation time and propagate $\mu_{\ell+1}=\exp(\Delta t_\ell Q_{t_\ell}^\top)\mu_\ell$. The experiment-specific grids and refinement checks are in Appendix~\ref{app:experiments}. This deterministic calculation is different from drawing trajectories.

    For an optional trajectory with a piecewise-constant generator, at state $i$ and time $t$ draw $\tau\sim\operatorname{Exp}(\lambda_i)$, where $\lambda_i=-Q_t(i,i)$. If $\lambda_i=0$, remain at $i$ until the next rate-change boundary. If $t+\tau$ reaches or crosses that boundary, advance to the boundary without a jump and restart using the new rates. Otherwise choose the first-jump destination with probability $Q_t(i,j)/\lambda_i$, advance to $t+\tau$, and repeat \citep{anderson2012continuous}. This is exact for the frozen generator and a time-discretization approximation for continuously varying rates.

\section{Extended discussions}
\label{app:extended_discussions}

\subsection{Extended related work}
\label{app:extended_related_work}

This section provides additional context for the methods summarized in Section~\ref{sec:related_work}. We first discuss extensions of Action Matching and discrete Wasserstein gradient-flow methods, and then distinguish DAM from discrete diffusion, score-based, and flow-based generative models.

\paragraph{Extensions of Action Matching.}
Beyond the fixed-path formulation of \citet{neklyudov2023action}, Wasserstein Lagrangian Flows additionally optimize the probability path subject to observed marginal constraints \citep{neklyudov2024wasserstein}, while \citet{berman2024parametric} consider parameter-dependent reduced dynamics. Wasserstein Hamiltonian formulations describe kinetic evolution directly on probability space \citep{chow2020wasserstein}, and Wasserstein Lagrangian Mechanics extends marginal-based learning to second-order population dynamics \citep{guan2026call}. Connections between Action Matching and quadratic optimal transport under optimal vector-field parameterizations are studied by \citet{kornilov2025equivalence}. These developments concern continuous state spaces, whereas DAM focuses on the additional density-dependent geometry that appears on finite graphs.

\paragraph{Discrete Wasserstein geometry.}
The discrete transport geometry used by DAM originates from constructions that endow finite-state probability spaces with mobility-dependent dynamical transport metrics. \citet{maas2011gradient} constructs a metric under which the heat equation of a reversible Markov chain is the gradient flow of relative entropy, while \citet{chow2012fokker} derive related discrete Fokker--Planck gradient flows on weighted graphs. \citet{mielke2011gradient} studies generalized finite-state gradient structures. Subsequent developments include entropy dissipation in \citet{chow2018entropy}, Riemannian formulations in \citet{li2018natural,li2022transport}, and the study of geodesics and their support in \citet{gangbo2019geodesics,erbar2019geometry}. These results provide the graph gradient, continuity equation, and kinetic-energy metric underlying DAM.

Discrete Wasserstein geometry has also been used computationally for objectives different from population-dynamics inference. \citet{diazgarcia2021population} connect Wasserstein gradient flows with population-game dynamics, while \citet{sun2023discrete} use discrete gradient flows to construct samplers for prescribed target distributions. In these settings the driving objective is given, whereas DAM infers a canonical velocity from observed temporal marginals.

\paragraph{Inverse Wasserstein gradient flows.}
A related inverse problem assumes that the observed population evolution is generated by the Wasserstein gradient flow of an unknown energy functional. In continuous spaces, \citet{bunne2022proximal} formulate this problem through repeated JKO steps and Brenier maps \citep{brenier1991polar}. \citet{terpin2024learning} derive a first-order alternative based on precomputed optimal-transport information, while \citet{persiianov2026learning} remove this preprocessing through a saddle-point formulation over the energy and transport maps.

\citet{rancati2026learning} extend the first-order JKO approach to finite graphs using discrete Wasserstein geometry. Their construction recovers an autonomous free-energy functional after computing graph transport geodesics between observed marginals. DAM instead treats the marginal evolution itself as prescribed and recovers its canonical minimum-action velocity, which may depend explicitly on time. The two methods therefore impose different structural assumptions: inverse JKO learning seeks one energy functional explaining the complete evolution, whereas DAM selects the minimum-action tangent compatible with the marginal path.

A deterministic map cannot generally push an atomic source to an arbitrary target: a point mass cannot map to a two-point mixture. Thus a map-based inverse-JKO formulation \citep{persiianov2026learning} needs adaptation when transport splits mass. Graph fluxes provide one alternative; \citet{li2018natural} discuss the distinct dynamical graph geometry. This does not exclude every finite-state analogue of a transport map.

\paragraph{Discrete diffusion and score-based models.}
Discrete generative modeling provides a separate use of finite-state stochastic dynamics. Multinomial diffusion and D3PM construct prescribed corruption processes whose reversal is learned from data \citep{hoogeboom2021argmax,austin2021structured}, while continuous-time formulations represent forward and reverse dynamics through continuous-time Markov chains \citep{campbell2022continuous}. Concrete Score Matching \citep{meng2022concrete} and SEDD \citep{lou2024discrete} use neighboring probability ratios as discrete analogues of continuous scores or as sufficient quantities for reverse-time dynamics.

DAM uses the same type of local density-ratio information for a different purpose. The ratios do not parameterize the population velocity or a reverse diffusion process; they determine the marginal-dependent mobility needed to evaluate the discrete Wasserstein kinetic energy. A separate potential then represents the canonical population dynamics. Appendix~\ref{app:discrete_scores} develops this connection in detail.

\paragraph{Discrete flow-based generative models.}
Discrete Flow Matching \citep{gat2024discrete} and related multimodal generative-flow constructions \citep{campbell2024generative} learn dynamics along prescribed or designed probability paths between endpoint distributions. Other approaches construct flows on the probability simplex using Dirichlet \citep{stark2024dirichlet} or Fisher--Rao geometry \citep{davis2024fisher}. Their primary goal is generative transport between endpoint distributions, with the intermediate probability path chosen as part of the modeling procedure. DAM instead addresses population-dynamics inference, where intermediate temporal marginals are observed data and the objective is to recover the canonical graph velocity associated with their evolution.

    \paragraph{Graph transport computation and population geometry.}
    \citet{lavenant2018surfaces} develop dynamical transport on discrete surfaces, and \citet{erbar2020computation} give computational methods for dynamical transport on graphs. \citet{gentile2026barycenters} study static and dynamic graph barycenters. These works optimize paths or distances between supplied distributions. The quadratically regularized graph transport problem of \citet{essid2018quadratic} uses linear ground distance, a different objective from DAM's kinetic action. \citet{peyre2015entropic} uses entropic regularization of transport plans to compute gradient flows; this regularization is distinct from entropy in a learned free energy.

    \citet{diazgarcia2021population} connect Wasserstein gradient flows to population games, while \citet{mertikopoulos2018riemannian} study game dynamics through Riemannian costs of motion. The simplex course of \citet{pistone2019information} supplies information-geometric background; its Fisher geometry differs from the graph-Wasserstein metric used here. \citet{erbar2014gradient} develops a nonlocal jump-kernel construction on Euclidean space, providing broader context rather than the direct finite-state source of each displayed operator.

    \paragraph{Attribution and the scope of the geometry.}
    We use probability coordinates and a prescribed symmetric conductance, as in graph formulations associated with Chow and collaborators; Maas's reversible-chain construction uses stationary-relative density, and Mielke develops finite-state gradient structures. These are connected constructions, not an exhaustive division into two schools. The conversion and its limitations are in \eqref{eq:maas_conductance_probability}. Geodesic work \citep{gangbo2019geodesics,erbar2019geometry} studies global and boundary questions; our action identity only uses the positive interior tangent problem and does not assume one entire admissibility framework contains another. More specifically, \citet[Theorem~3.4]{erbar2019geometry} give a Hamilton--Jacobi duality proof building on the earlier duality result of \citet{gangbo2019geodesics}. The former use an irreducible reversible Markov triple and an admissible mean; the latter formulate symmetric graph weights and a smooth, concave, homogeneous interior-positive mobility with finite $\int_0^1g(r,1-r)^{-1/2}\,dr$, treating geodesics that can meet the boundary. Their global results require more than the interior assumptions used here.

    \paragraph{Path optimization and mechanics.}
    The WLF formulation of \citet{neklyudov2024wasserstein} expresses a primal action over probability paths through a dual problem: the action potential is fitted for a path, and an outer objective optimizes the path. Differentiation through a parameterized path is distinct from backpropagating through a simulated trajectory. We make no historical ``first'' claim for that dual principle. \citet{guan2026call} instead introduce Wasserstein Lagrangian Mechanics to learn second-order population evolution from snapshots; this is a distinct objective from WLF and is not a directly interchangeable finite-state baseline. \citet{kornilov2024optimal} study the optimal-transport connection for Flow Matching (FM), whereas \citet{kornilov2025equivalence} establish the corresponding AM equivalence.

    \paragraph{Hamiltonian and accelerated sampling.}
    \citet{zhou2026accelerated} use mobility-weighted graph kinetic energy in a damped Hamiltonian system for a probability distribution and momentum potential, with a jump-process realization. DAM shares a potential-based continuity equation after matching conductances and signs, but fits a prescribed path rather than imposing their additional momentum equation to approach a target distribution. This relationship does not turn the learned action potential into a target log density for classical Hamiltonian Monte Carlo, nor imply accelerated mixing or greater sampling accuracy. Such guarantees require a separately specified dynamics and analysis.

    \paragraph{Time-discrete and aggregate-observation population inference.}
DICE \citep{blickhan2025dice} learns population dynamics on continuous state spaces using an objective derived from a time-discrete weak continuity equation. Its loss is invariant under spatially constant, time-dependent shifts of the potential. This discretization concerns time, whereas DAM provides a finite-state counterpart of Action Matching under a chosen graph mobility. With finite observations, DAM explicitly specifies the interpolating marginal path used by its learning objective.

\citet{mascherpa2025aggregate} estimate a common Markov transition matrix jointly with transport plans between aggregate observations through a convex entropic optimal-transport formulation. DAM instead learns the canonical, potentially time-dependent velocity of a prescribed marginal path under its chosen mobility. Its local density-ratio representation supplies the marginal information needed to evaluate the action from samples without graph-geodesic preprocessing. The FEGF adaptations used in our experiments, including the time-indexed potential for gastrulation, are specified in Appendix~\ref{app:implementation}.

\paragraph{Potential constructions and graph bridges.}
Potential representations and the associated Riemannian formulation on the probability simplex are developed by \citet{chow2019discrete}. In discrete generative modeling, \citet{shaul2025kinetic} construct kinetic-optimal flows using graph-Laplacian potentials along prescribed or designed probability paths. \citet{campbell2024generative} show how detailed-balance additions to transition rates preserve a prescribed marginal evolution. These works establish closely related potential constructions and the freedom to choose among compatible Markov realizations. DAM combines these established structures with a mobility-dependent action objective and local-ratio estimation for a training path specified from unpaired temporal snapshots.

Generalized Schr\"odinger bridges on graphs \citep{theodoropoulos2026bridge} optimize path laws under endpoint constraints, a relative-entropy penalty against a reference process, and additional running costs. DAM fits a prescribed marginal path and graph mobility.

\subsection{Cortical prediction from unpaired observations}
\label{app:cortical_prediction}
Graph-based studies model tau propagation through brain connectivity and its interaction with amyloid burden \citep{vogel2020spread,lee2022regional}, while X-FlowNet and NeuralOMT learn transport-related dynamics from longitudinal regional or cortical measurements \citep{dan2023rethink,huang2024cortical}. An interesting direction is to combine these models with the unpaired marginal objective of DAM \textbf{(ours)}, potentially exploiting observations without matched temporal samples and reducing dependence on complete longitudinal records. Such training could complement paired supervision and improve data utilization; whether it improves next-visit accuracy remains to be tested. The present formulation does not by itself identify subject-specific transitions from unmatched cohorts. Unpaired state samples are not automatically unpaired subjects: personalized prediction additionally requires a justified conditional population model and registered, common state labels.

The sphere and hand experiments provide an intermediate step from controlled synthetic dynamics to cortical surface modeling. To extend this progression to cortical biomarker evolution with changing total burden, we propose developing unbalanced action matching with source and sink terms, alongside conditioning on baseline information and addressing cohort and disease-stage differences. Normalizing PET maps discards absolute burden, and balanced probability transport alone cannot model its increase; NeuralOMT already permits variable mass \citep{huang2024cortical}. Evaluation should use held-out longitudinal subjects and distinguish predictive accuracy from biological identification of inferred transport. At inference, an unavailable follow-up or future target burden must not enter the closure, initialization or rescaling. Fitting observed times does not establish extrapolation, and neither improved prediction nor clinical benefit follows from the present experiments.

\section{Additional experimental details}
\label{app:experiments}

We describe the synthetic reconstruction and gastrulation interpolation benchmarks, compare mobility and density-ratio choices, and evaluate transport on two surface meshes.

\subsection{Finite-snapshot synthetic protocol}
\label{app:fegf_sample_only_protocol}

\paragraph{Data and graphs.}
We adapt the graph free-energy benchmark of \citet{rancati2026learning} to five sampled marginals at times $0,1/4,1/2,3/4,1$. Paths follow free-energy descent with energy $F(p)=\sum_i p_iV_i+\beta\sum_i p_i\log(p_i/\pi_i)$, where $\pi$ is the stationary distribution and $V_i\sim\operatorname{Unif}[-1,1]$ independently. Initial relative densities $\rho_i=p_i/\pi_i$ are sampled uniformly on $[0,100]$ and normalized to satisfy $\sum_i\pi_i\rho_i(0)=1$.

We use eleven graph families: complete, Erd\H{o}s--R\'enyi, regular, Watts--Strogatz, stochastic block model, Delaunay, Euclidean minimum spanning tree, grid, $k$-partite, torus and Apollonian. Geometric constructions use jitter $0.05$ and an edge-weight length scale equal to $0.25$ times the mean edge length. Row-normalizing symmetric weights gives the transition matrix $K$; normalized weighted degrees give $\pi$. DAM uses edge weights $\omega_{ij}=\pi_iK_{ij}$.

The base setting is $(N,M,\beta)=(60,1000,0.1)$, where $N$ is the number of states and $M$ the samples per marginal in each split. One-factor sweeps vary $\beta\in\{0.01,0.1,0.2\}$, $N\in\{20,40,60,80,100\}$ or $M\in\{500,1000,2000,5000,10000\}$. Their union contains eleven settings, giving $363$ graph/setting/seed combinations across seeds 42--44. Methods receive matched, independently sampled training, validation and test splits. Smaller sample budgets use prefixes of the same draws. Fitting and model selection use sampled observations only, without access to the generating probabilities, potentials or intermediate states.

\paragraph{Numerical generation.}
For $\beta=0.1$, we use adaptive Dormand--Prince integration. For $\beta\in\{0.01,0.2\}$, we use backward Euler in log-density coordinates with steps at most $10^{-4}$, checked against $5\cdot10^{-5}$. Generated paths must have positive finite densities, mass error at most $10^{-8}$, nonincreasing energy and observation-probability TV changes at most $10^{-4}$ under refinement.

\paragraph{Evaluation.}
We report Hellinger distance $H(p,q)=2^{-1/2}\bigl[\sum_i(\sqrt{p_i}-\sqrt{q_i})^2\bigr]^{1/2}$ and total variation $\operatorname{TV}(p,q)=\tfrac12\sum_i|p_i-q_i|$. All methods start from the training empirical marginal at $t=0$, propagate without resets and are scored against independent test histograms at the four subsequent observation times. This measures reconstruction at observed times and includes test-sampling noise. We average errors over times within each run and over graphs/settings within each seed, then report means and sample standard deviations across three seeds.

DAM propagates marginals through frozen-generator matrix exponentials with eight substeps per observation interval for both validation and testing (Appendix~\ref{app:dam_empirical_and_sampling}). Its softplus rate split uses $1/\varepsilon=1$. FEGF propagates its predicted law in stationary-relative coordinates. Both comparisons use deterministic marginal propagation.

\subsection{Training and baseline specification}
\label{app:implementation}

\begin{table}[t]
    \centering\small
    \caption{Synthetic and gastrulation training settings. All models use Adam, float64 arithmetic, weight decay $10^{-6}$, gradient-clipping norm $1$ and $200$ updates per epoch.}
    \label{tab:training_settings}
    \begin{tabularx}{\linewidth}{@{}p{0.20\linewidth}XX@{}}
        \toprule
        Setting & DAM & FEGF \\
        \midrule
        Architecture & 16-dimensional node embedding, scalar time, two width-64 ReLU layers, scalar output & Synthetic: edge network with embedding 16 and width 64; gastrulation: time-indexed scalar table \\
        Learning rate & $5\cdot10^{-4}$ & Synthetic: $5\cdot10^{-4}$; gastrulation: $10^{-2}$ \\
        Batch size & 128 & 64 \\
        Maximum epochs & Ratio: 60; action: 100 & 100 \\
        Selection & Ratio validation loss; action observed-time validation TV & Observed-time validation TV \\
        Early-stopping patience & Ratio: 10 epochs; action: none & 15 epochs \\
        Support treatment & Uniform-reference mixing $\delta=0.01$ & Pseudocount $0.5$ per state; rollout mixture $0.98p+0.02\pi$ \\
        \bottomrule
    \end{tabularx}
\end{table}

Both DAM stages use the same smoothstep interpolant with uniform-reference mixing, including at the endpoints, and sample four neighbors per state with importance weights. Histogram DAM evaluates $U(t,i)=\log\widetilde p_i(t)$; learned-ratio DAM minimizes the local softplus loss. Both variants share action architectures, initial weights, training streams, update budgets and selection rules. Potentials have zero spatial mean. ReLU potentials are continuous and piecewise differentiable, and we use their almost-everywhere time derivatives. Table~\ref{tab:training_settings} gives the fixed settings, reused across gastrulation folds without a new tuning search.

\paragraph{FEGF adaptation.}
We adapt the free-energy gradient-flow method (FEGF) of \citet{rancati2026learning}. Empirical probabilities receive pseudocount $0.5$ and are converted to $\rho=p/\pi$. Adjacent-interval targets use the actual time gap, a weighted-incidence least-squares solve with one pinned vertex, and the negative free-energy-difference convention. Logarithmic means use an input floor of $10^{-6}$, and the continuity residual is checked. The last interval target is reused at the terminal observation.

Synthetic FEGF learns edge differences and a positive entropy coefficient. Gastrulation FEGF instead learns a scalar value per state and observation time, interpolated in time, with no trained entropy coefficient. Both minimize the sample-based velocity-target loss and select checkpoints by observed-time validation TV under propagation without resets. The omitted atlas marginal is used only for scoring.

FEGF uses conductances $\pi_iK_{ij}\theta_{\log}(p_i/\pi_i,p_j/\pi_j)$, whereas logarithmic DAM uses $\pi_iK_{ij}\theta_{\log}(p_i,p_j)$. These geometries differ when $\pi$ is nonuniform. The comparison evaluates marginal prediction under these respective geometries; fixed optimizer budgets do not imply equal prior tuning effort.

\subsection{Synthetic results and ablations}
\label{app:fegf_topology_results}
\label{app:fegf_ablations}

Figure~\ref{fig:rancati_topology} reports the base setting, while Figure~\ref{fig:rancati_ablations} varies one factor at a time. Both DAM variants achieve lower mean Hellinger errors than FEGF across all base-setting graph families and the reported sweep averages. Learned ratios and histograms perform comparably, with no consistent advantage for learned ratios across metrics.

\begin{figure}[tbp]
    \centering
    \includegraphics[width=\linewidth]{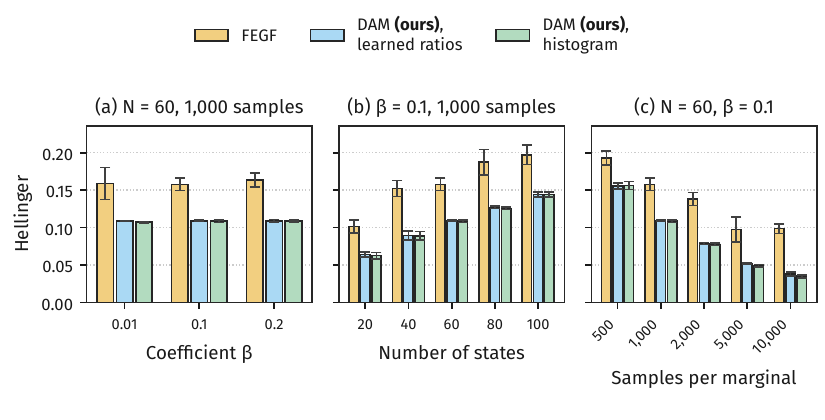}
    \caption{Synthetic sweeps with logarithmic mobility. Bars average Hellinger errors over eleven graph families; whiskers show sample SD across three seeds. Each panel varies one factor from the base setting. Smaller sample budgets share draws.}
    \label{fig:rancati_ablations}
\end{figure}

\paragraph{Mobility and density-ratio comparisons.}
We refit each DAM variant with arithmetic, geometric, harmonic and logarithmic mobility, holding data, smoothing, interpolation, architecture and optimizer budgets fixed. Table~\ref{tab:mobility_closure} reports both metrics, and Figure~\ref{fig:closure_mobility} shows seed variability. Arithmetic action training is algebraically ratio-free, although its Markov sampler still uses density ratios. Harmonic mobility gives the lowest gastrulation means among the tested choices; it is a descriptive comparison, without independent mobility selection.

\begin{table}[t]
    \centering\small
    \caption{Mean errors for separately fitted mobilities. Synthetic results average eleven base-setting graphs and three seeds; gastrulation results average seven omitted stages and three seeds. L: learned ratios; Hst: histograms. Lower is better; values are rounded to three decimals.}
    \label{tab:mobility_closure}
    \begin{tabular}{llrrrr}
        \toprule
        Dataset & Mobility & $H$ (L) & $H$ (Hst) & TV (L) & TV (Hst) \\
        \midrule
        Synthetic & Arithmetic  & 0.110 & 0.110 & 0.106 & 0.107 \\
                  & Geometric   & 0.109 & 0.108 & 0.106 & 0.106 \\
                  & Harmonic    & 0.109 & 0.108 & 0.106 & 0.106 \\
                  & Logarithmic & 0.109 & 0.109 & 0.106 & 0.106 \\
        \midrule
        Gastrulation & Arithmetic  & 0.285 & 0.262 & 0.268 & 0.242 \\
                     & Geometric   & 0.234 & 0.229 & 0.220 & 0.215 \\
                     & Harmonic    & 0.220 & 0.203 & 0.207 & 0.188 \\
                     & Logarithmic & 0.250 & 0.251 & 0.232 & 0.233 \\
        \bottomrule
    \end{tabular}
\end{table}

\begin{figure}[tbp]
    \centering
    \includegraphics[width=\linewidth]{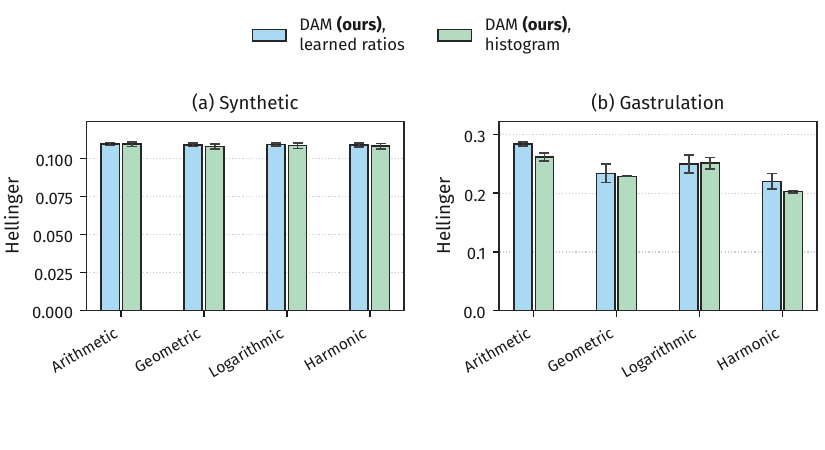}
    \caption{Hellinger errors for learned-ratio and histogram DAM, refitted for each mobility. Bars and whiskers show means and sample SD across three seeds after averaging graphs or omitted stages. Panels use separate vertical scales.}
    \label{fig:closure_mobility}
\end{figure}

\subsection{Mouse gastrulation}
\label{app:gastrulation_experiment}

\paragraph{Data and graph.}
We use processed counts associated with the mouse gastrulation atlas of \citet{pijuan2019single}: $89{,}267$ cells, $34$ cell types and nine stages from E6.5 to E8.5 at quarter-day intervals. These processed counts define the benchmark; the exact source release and filtering steps have not been fully recovered.

Atlas marginals use pseudocount smoothing, $p_{ki}=(n_{ki}+0.5)/(n_k+17)$. DAM additionally mixes each resampled empirical marginal with a uniform reference, since resampling can omit cell types. FEGF applies the separate support treatments in Table~\ref{tab:training_settings}. The graph combines a connected lineage prior with $44$ unordered edges and a background weight of $0.01$ added to every adjacency entry before normalization. All $561$ unordered state pairs consequently support transport, including transitions outside the lineage prior.

\paragraph{Leave-one-stage-out interpolation.}
For each of the seven interior stages, all methods are fitted to the same eight remaining stages, including later observations. Times are normalized by $(t-6.5)/2$. The omitted stage is removed before constructing training inputs and FEGF targets. Each retained stage supplies $10{,}000$ training, $5{,}000$ validation and $5{,}000$ independent test draws; checkpoint selection uses retained-time validation only.

Each prediction starts from the training empirical marginal at the immediately preceding retained stage and is scored against the omitted, pseudocount-smoothed atlas marginal. This evaluates interpolation on the measured atlas. The empirical initial law differs from the smoothed training interpolant, so the exact-initialization condition of the Markov realization need not hold.

\paragraph{Numerics and results.}
All methods propagate deterministic marginals. Without consulting target errors, we compare $128/256$ substeps and double up to $1024/2048$ until the prediction-TV change is at most $0.001$. All $21$ stage/seed cases pass. Scores average the seven stages equally; standard deviations use three stage-averaged seed values and reflect resampling and optimization of the same atlas.

With harmonic mobility, both DAM variants achieve lower mean Hellinger and TV errors than FEGF and persistence, which keeps the preceding empirical marginal unchanged. Figures~\ref{fig:gastrulation_loo} and~\ref{fig:gastrulation_tv} show stagewise Hellinger and TV errors. Differences between learned ratios and histograms depend on the stage. Discrepancy from the training interpolant also reflects initialization and fitting errors, which remain after numerical refinement.

\begin{figure}[tbp]
    \centering
    \includegraphics[width=\linewidth]{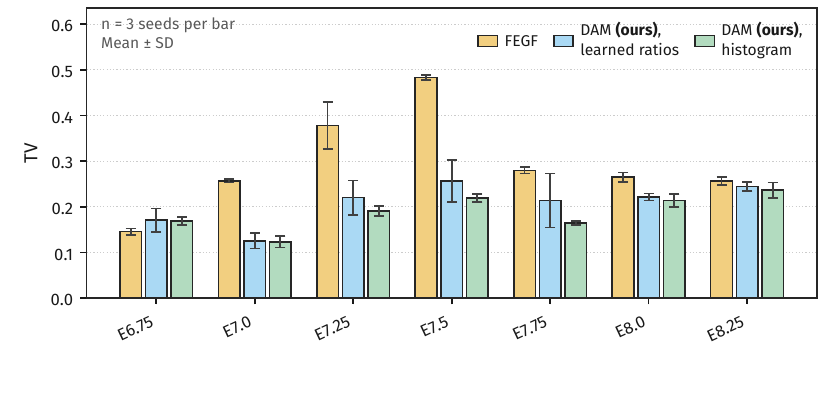}
    \caption{Gastrulation TV errors against omitted, smoothed atlas marginals. Each fold uses eight retained stages, including later observations. DAM uses harmonic mobility; FEGF uses its stationary-relative logarithmic geometry. Whiskers show sample SD across three sampling/optimization seeds, rather than biological replicates.}
    \label{fig:gastrulation_tv}
\end{figure}

\subsection{Sphere and hand surface transport}
\label{sec:mesh_experiment}

\begin{wraptable}{R}{0.50\linewidth}
    \centering\small
    \caption{Mean TV at $16$ omitted times against independent test-particle marginals, seed 42; lower is better.}
    \label{tab:mesh_accuracy}
    \setlength{\tabcolsep}{3pt}
    \begin{tabularx}{\linewidth}{@{}Xrr@{}}
        \toprule
        Method & Sphere & Hand \\
        \midrule
        DAM, histogram & 0.169 & 0.246 \\
        DAM, local + global ratios & 0.192 & 0.250 \\
        Tuned FEGF & 0.564 & 0.843 \\
        \bottomrule
    \end{tabularx}
\end{wraptable}

We evaluate DAM on a sphere with $1020$ vertices and a hand mesh with $1515$ vertices, using marginal paths supplied by the ADMM surface-transport solver of \citet{lavenant2018surfaces}. The learner's graph geometry differs from the teacher's triangle-based geometry.

Each mesh uses $50{,}000$ training, $25{,}000$ validation and $25{,}000$ independent test particles, with seed 42. Here DAM uses paired vertex trajectories derived from the teacher's continuity equation. Both DAM variants obtain lower omitted-time TV than tuned FEGF, with histogram DAM competitive or better (Table~\ref{tab:mesh_accuracy}). These results concern two supplied paths and one seed; transfer to new endpoint pairs or unseen meshes remains untested.

\subsection{Surface teacher, sampling and learning protocol}
\label{app:mesh_protocol}

\paragraph{Teacher and paired sampling.}
The ADMM solver uses $31$ time cells, congestion coefficient $0.02$ and residual tolerance $10^{-4}$. Its cell-centered marginals and endpoints form $33$ saved times. Alternating times, including both endpoints, provide $17$ DAM observations; the remaining $16$ are excluded from DAM fitting and selection.

We convert finite-element momentum to minimum-norm triangle-edge flows preserving its divergence, then correct the finite ADMM continuity residual by a sparse graph projection. Conversion requires residual TV at most $0.005$ and an edge-correction $\ell^1$ norm at most $10^{-4}$ plus $5\%$ of the original flow norm on each interval. An area-reference mixture of $0.001$ ensures positive sampling probabilities, and ODE verification checks accumulated sampling TV against $0.0005$. This produces paired vertex Markov trajectories without identifying continuous surface characteristics. DAM receives no teacher velocities; all rollouts start from the same training empirical marginal and propagate without resets.

\paragraph{DAM training.}
Both variants use logarithmic mobility, smoothstep interpolation and reference mixing $0.01$. Histogram DAM interpolates training histograms before taking logarithms and accesses this stored path at inference. Learned-ratio DAM combines local and uniformly proposed global-pair logistic losses with equal weights. Both train the potential using paired integration-by-parts samples, preserving the population objective.

Networks use $32$-dimensional node embeddings, width-$64$ hidden layers, scalar time input and a zero-mean potential gauge. Ratio and action training use $60$ and $100$ epochs, respectively, with $200$ updates per epoch. In the same order, learning rates are $2.2846\cdot10^{-3}$ and $9.4677\cdot10^{-4}$, weight decays $1.1117\cdot10^{-5}$ and $3.7445\cdot10^{-8}$, batch sizes $512$ and $128$, and clipping norms $0.25935$ and $3.39333$. The variants share potential architectures, initialization and budgets. Hand models reuse sphere settings with fresh weights.

Checkpoints minimize the respective simulation-free validation objectives on retained-time interpolants, using eight Gauss--Legendre nodes per interval and a final order-$16$ check. Validation uses independent samples; action losses under different frozen ratio models are not ranked against one another.

\paragraph{FEGF adaptation and selection.}
Surface FEGF uses a time-indexed scalar potential table, empirical geodesic targets and linear potential interpolation, with no trained entropy coefficient. Separate searches complete $93$ sphere trials and $102$ hand trials, with at most $200$ epochs each. Every five epochs, selection evaluates TV at the $16$ omitted validation times using $16$ rollout substeps and $30$-epoch patience. Thus FEGF and DAM use different validation information, geometries and tuning budgets; test data are excluded from selection.

Selected sphere and hand learning rates are $0.0587579$ and $0.0582306$, weight decays $2.87534\cdot10^{-8}$ and $0.000138841$, batch sizes $256$ and $128$, clipping norms $1.75120$ and $0.138872$, updates per epoch $400$ and $200$, and mass floors $0.000144722$ and $0.167839$, respectively.

\paragraph{Rollouts and visualization.}
Step doubling must change predictions by at most $0.001$ in TV, checking numerical consistency rather than absolute integration error. DAM uses $32$ midpoint-frozen substeps per interval; tuned FEGF uses $1024$ on sphere and $32$ on hand, while fixed FEGF uses $256$ on both. The accuracy table averages the $16$ omitted times; diagnostics and costs use all $32$ noninitial times.

Figures~\ref{fig:sphere_transport} and~\ref{fig:hand_transport} show probability divided by normalized vertex area, with face colors averaging vertex densities. Each mesh uses a fixed camera and logarithmic color scale shared across methods and times. Quarter-time marginals are interpolated for display only; metrics use the original times.

\begin{figure}[tbp]
    \centering
    \includegraphics[width=\linewidth]{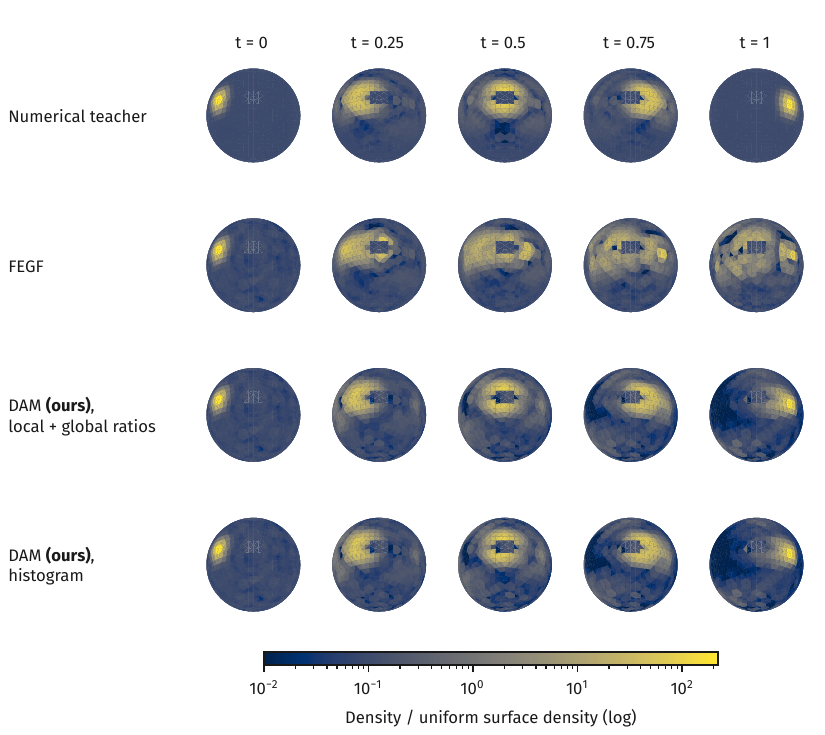}
    \caption{Sphere transport, seed 42. Rows: numerical teacher, tuned FEGF, DAM with local + global ratios, and histogram DAM. Columns: $t=0,0.25,0.5,0.75,1$. Colors show area-normalized probability density.}
    \label{fig:sphere_transport}
\end{figure}

\begin{figure}[tbp]
    \centering
    \includegraphics[width=\linewidth]{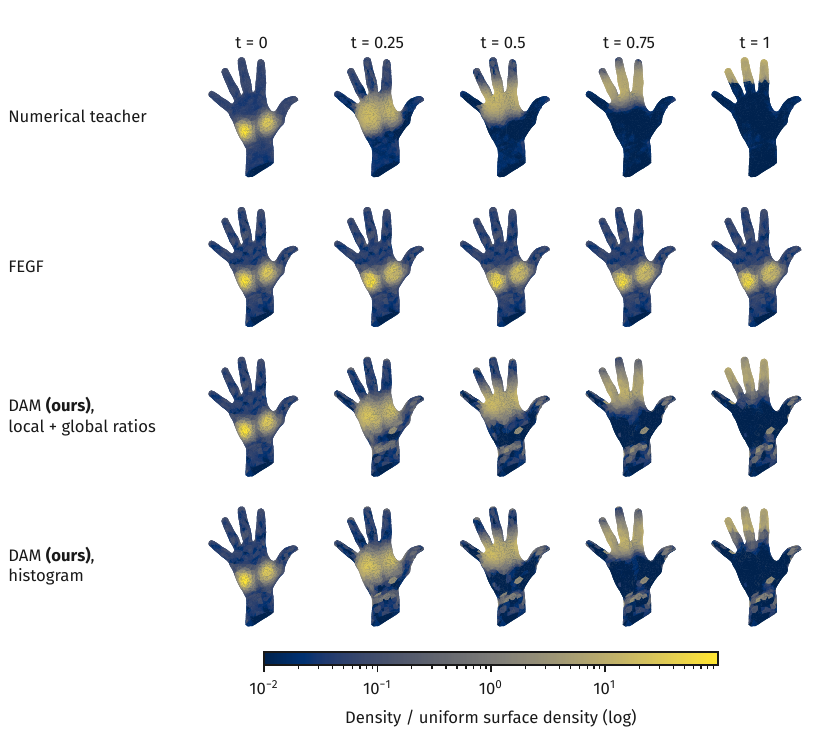}
    \caption{Hand transport, seed 42, using the row order, display times and density-rendering convention of Figure~\ref{fig:sphere_transport}.}
    \label{fig:hand_transport}
\end{figure}

\subsection{Mesh estimator and integration diagnostics}
\label{app:mesh_diagnostics}

In a cumulative sphere study, paired integration-by-parts training improves rollout accuracy. Adding global-pair ratio estimation improves density and endpoint errors but worsens mean trajectory TV. Better density estimates therefore do not uniformly improve rollouts. For the combined model, midpoint freezing reduces the substeps needed to meet the $0.001$ refinement tolerance from $128$ to $32$, with little change in accuracy. These single-seed comparisons do not isolate independent causal effects, and numerical refinement does not control sampling or fitting error.

\subsection{Accuracy versus computational cost}
\label{app:mesh_cost}

Figure~\ref{fig:mesh_cost} reports CPU times on an Intel Xeon Gold 6338 using float64 arithmetic and one thread. Medians and ranges cover three repetitions of configurations using seed 42, so whiskers measure runtime variation.

DAM training includes setup and potential fitting, plus ratio fitting for the learned variant. Fixed FEGF includes target preparation, optimization and validation but excludes model setup. Both exclude teacher generation, particle-data creation and hyperparameter search. Tuned FEGF appears only for inference because its training time was not measured. Inference covers $33$ output times, excluding model loading and output serialization.

Training exceeds a single teacher solve in these runs. Since the methods have unequal errors, histogram DAM accesses stored intermediate distributions, and endpoint-constrained transport differs from evaluating a learned path, these timings do not establish an equal-accuracy or end-to-end speedup on new transport problems.

\begin{figure}[tbp]
    \centering
    \includegraphics[width=\linewidth]{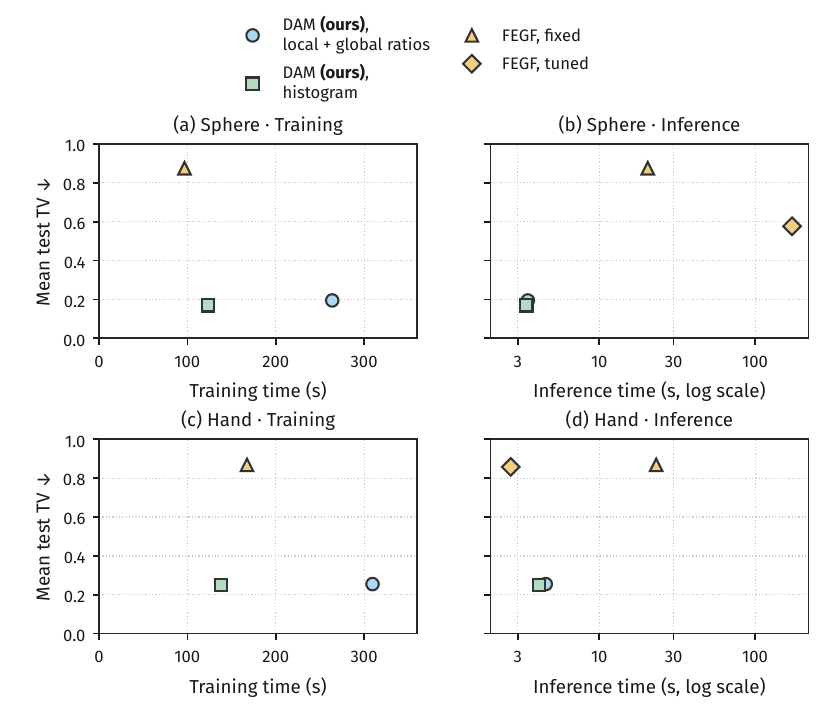}
    \caption{Accuracy versus CPU cost on sphere (top) and hand (bottom): training (left) and marginal inference (right, logarithmic time axis). TV averages $32$ noninitial times. Whiskers span three runtime repetitions around median times. Tuned FEGF is inference-only.}
    \label{fig:mesh_cost}
\end{figure}

\end{document}